\documentclass[sigconf, nonacm]{acmart}

\usepackage{amsfonts, amsmath}
\usepackage{xcolor}
\usepackage{subcaption}
\usepackage{amsthm}
\usepackage[capitalise]{cleveref}
\usepackage{enumitem}
\usepackage{makecell}
\usepackage{pifont}

\DeclareMathOperator*{\argmax}{argmax}
\DeclareMathOperator*{\argmin}{argmin}

\newtheorem{theorem}{Theorem}
\newtheorem{remark}{Remark}
\newtheorem{definition}{Definition}
\newtheorem{assumption}{Assumption}

\newtheorem{claim}[theorem]{Claim}

\newcommand{\algname}{\textit{Item-Weighted PCA}}

\newcommand{\bR}{\mathbb{R}}
\newcommand{\hX}{\hat{X}}

\AtBeginDocument{%
  }

\begin{document}

\title{When Collaborative Filtering is not Collaborative: Unfairness of PCA for Recommendations}

\author{David Liu}
\orcid{0000-0002-2129-447X}
\affiliation{%
  \institution{Northeastern University}
  \city{Boston}
  \state{MA}
  \country{USA}
}
\email{liu.davi@northeastern.edu}

\author{Jackie Baek}
\orcid{0000-0001-5538-509X}
\affiliation{%
  \institution{New York University}
  \department{Stern School of Business}
  \city{New York}
  \state{NY}
  \country{USA}
}
\email{baek@stern.nyu.edu}

\author{Tina Eliassi-Rad}
\orcid{0000-0002-1892-1188}
\affiliation{%
  \institution{Northeastern University}
  \city{Boston}
  \state{MA}
  \country{USA}
}
\email{t.eliassirad@northeastern.edu}

\begin{abstract}
We study the fairness of dimensionality reduction methods for recommendations. We focus on the fundamental method of principal component analysis (PCA), which identifies latent components and produces a low-rank approximation via the leading components while discarding the trailing components. Prior works have defined notions of ``fair PCA''; however, these definitions do not answer the following question: \textit{why} is PCA unfair? We identify two underlying popularity mechanisms that induce item unfairness in PCA. The first negatively impacts less popular items because less popular items rely on trailing latent components to recover their values. The second negatively impacts highly popular items, since the leading PCA components specialize in individual popular items instead of capturing similarities between items. 
To address these issues, we develop a polynomial-time algorithm, \algname{}, that flexibly up-weights less popular items when optimizing for leading principal components. We theoretically show that PCA, in all cases, and Normalized PCA, in cases of block-diagonal matrices, are instances of \algname{}. We empirically show that there exist datasets for which \algname{} yields the optimal solution while the baselines do not. In contrast to past dimensionality reduction re-weighting techniques, \algname{} solves a convex optimization problem and enforces a hard rank constraint. Our evaluations on real-world datasets show that \algname{} not only mitigates both unfairness mechanisms, but also produces recommendations that outperform those of PCA baselines.
\end{abstract}

  
\maketitle

\begin{figure*}[ht]
    \centering
    \includegraphics[width=1.0\textwidth]{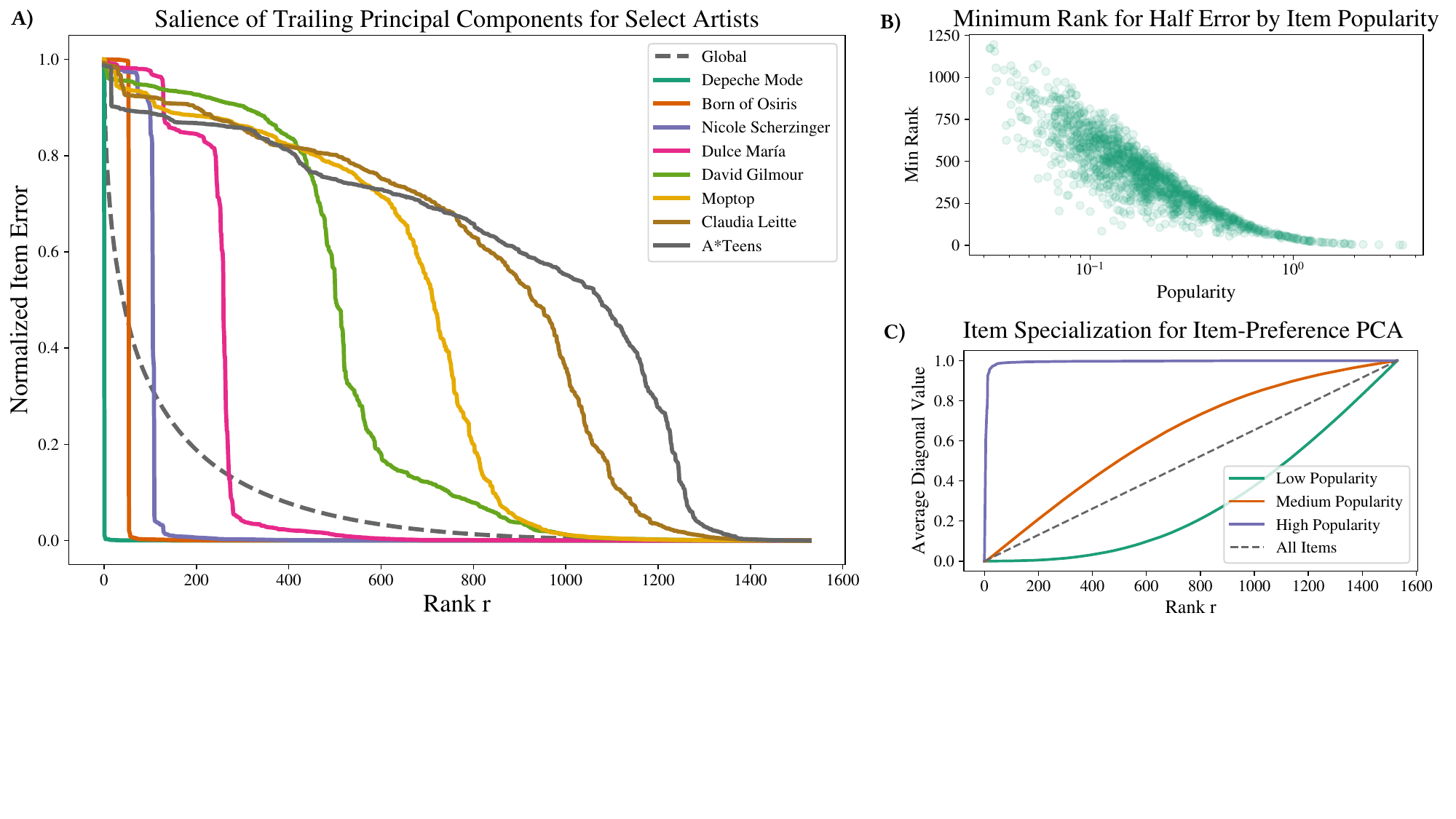}
    \Description{Collection of three plots. Plot A shows the Normalized Item-Error (y-axis) as a function of embedding rank (x-axis) for several artists. Each artist observes a drop in error at a specific rank value. Plot B is a scatter plot where each point is an artist; the x-axis is the popularity, and the y-axis is the rank at which the reconstruction error is halved. The points are in a downward sloping curve. Plot C shows the average diagonal value among high, medium and low popularity artists at various rank values. All curves start at the origin. As rank increases, the high popularity curve immediately increases while the curves for medium and low popularity are gradual.}
    \caption{
    To show the unfairness of PCA for recommendations, we run PCA on the LastFM dataset.
    Subfigure (A) shows the normalized item error as a function of the rank for eight different artists, as well as the overall error in the dotted line. While on average the PCA approximation exhibits diminishing returns as the rank increases, for individual items, specific components are critical for improving approximation quality.
    Subfigure (B) shows the first unfairness mechanism: less popular items rely on trailing principal components. The plot shows that high-popularity artists require fewer components to halve the initial approximation error while less-popular artists rely on trailing components.
    Subfigure (C) shows the second unfairness mechanism: PCA components specialize in individual items as opposed to collaborating across items. The diagonal values of PCA projection matrices indicate the degree of specialization.
    }
    \label{fig:teaser}
\end{figure*}

\section{Introduction}

The growing prevalence of machine learning algorithms in various real-world applications makes it important to understand the underlying mechanisms that drive the decision-making processes of these algorithms. In this context, our work focuses on a popular algorithm: Principal Component Analysis (PCA). Our goal is to understand the downstream effects of this algorithm. In doing so, we focus on identifying undesirable, systematic issues that may arise for the individuals affected by the algorithm's decisions. We use the term ``unfairness'' to refer to issues that have a negative or undesirable effect on an individual or a group of individuals.

PCA is a well-established dimensionality reduction technique widely used in many fields \citep{pearson1901liii,hotelling1933analysis}. PCA extracts key features from datasets by projecting them onto \textit{principal components}, thereby reducing the number of dimensions while preserving important information. PCA has many downstream applications, and what kind of ``unfairness issues'' exist depends heavily on the exact application. Therefore, we focus our work on one common application, namely recommendation systems.

\paragraph{Recommendation systems and collaborative filtering.}
We use the running example of the LastFM music platform, where users listen to music by various artists. (We refer to artists as \textit{items}.)
In this context, the goal of a recommendation system is to help users discover artists they would enjoy listening to.  Collaborative filtering (CF) is a popular approach for recommendations that relies on using data on user-item preferences to find patterns within these preferences. Dimensionality reduction methods, especially PCA, are commonly used for CF (e.g., \cite{goldberg2001eigentaste,kim2005collaborative,nilashi2015multi}). In this paper, we focus on the implications of using PCA for CF and, in particular, on identifying unfairness issues at the \textit{item}-level.

\paragraph{Contributions.}
We identify the mechanisms in the PCA algorithm that can negatively impact items in the context of recommendations and then develop an approach that combats these unfairness mechanisms. Our main contributions are as follows:
\begin{enumerate}
    \item We identify two mechanisms in which PCA may introduce an undesirable item-level impact within the context of CF.
    \begin{enumerate}
        \item The first mechanism is that the leading components of PCA often lack meaningful information related to less popular items.
        This may lead to fewer or worse-quality recommendations with respect to these less popular items.
        \item The second mechanism is that, in the existence of highly popular items, the leading components of PCA can each contain information about a \textit{single} popular item, rather than capturing similarities between items. Such components are not useful for the sake of CF, as they do not contain any ``collaborative'' information; this can adversely impact the recommendations related to the highly popular items.
    \end{enumerate}
    We demonstrate both of these mechanisms empirically through the LastFM dataset, summarized in \cref{fig:teaser}, as well as theoretically on a stylized class of matrices.
    \item We propose a computationally efficient algorithm called \algname{}, which up-weights less popular items. In contrast to past CF re-weighting techniques, \algname{} solves a convex optimization problem while also enforcing a hard rank constraint.
    \begin{enumerate}
        \item We consider two natural benchmark algorithms,  Vanilla PCA, as well as Normalized PCA, which normalizes each column of the matrix before performing PCA.
        We show that PCA is a specific instance of \algname{}, as is Normalized PCA in the case of item-regular block-diagonal matrices.
        \item We then show that the approach of setting the weights to be inversely proportional to an item's norm is an \textit{interpolation} between the two benchmark algorithms. For a stylized class of matrices, this interpolation balances popular and less popular items while the baselines do not. 
        We use this weighting procedure for all of our numerical experiments.
    \end{enumerate}
    \item We present empirical results demonstrating that \algname{} not only mitigates specialization but also yields recommendations that outperform those of PCA baselines.
\end{enumerate}

The rest of the paper is organized as follows: in \cref{sec:related_work}, we summarize the multiple areas of related work; in \cref{sec:unfairness_of_pca} we present the two unfairness mechanisms before introducing our algorithm \algname{} in \cref{sec:algorithm}. We present experimental evaluations in \cref{s:experiments} and discuss societal implications of our findings in \cref{sec:discussion}.
\section{Related Work} \label{sec:related_work}
Our work builds directly on a line of literature that examines the fairness of PCA. Below we summarize the findings of past fair PCA works and distinguish them from our motivation. Our algorithm also builds on a long literature of weighted dimensionality reduction and CF; below we summarize the purpose and properties of past weighted CF algorithms. Finally, our popularity unfairness mechanisms complement efforts to mitigate popularity bias in the recommender systems literature, but while many of these works focus on the outputs of recommender systems, we are focused on the internal collaboration mechanism.

\paragraph{\textbf{Fairness of PCA}}
The existing literature on fair PCA can be summarized as imposing a fairness constraint on the PCA problem and developing a new algorithm to satisfy this constraint.
Specifically, existing works assume that the set of users is partitioned into predefined groups (based on, e.g., race or gender). 
Several papers~\cite{fair-pca, fair-pca-multi-criteria, efficient-fair-pca, closed-form-fair-pca} define fairness as requiring the reconstruction error between groups of users to be ``balanced'', for different definitions of balance.
Alternatively, \citet{olfat2019convex} define the output of a PCA algorithm as fair if the group label cannot be inferred from the projected data, while \citet{mmd-fair-pca} aim to minimize the difference in the conditional distributions of the projected data. \cref{tab:alg-comparison} summarizes the differences between our work and existing literature. 
One difference is that prior works focus on \textit{user}-level fairness with pre-defined groups, whereas we focus on \textit{items}, with no reliance on group labels.

\begin{table}
\centering
\Description{Table summarizing past works on fair PCA.}
\caption{Comparison with existing papers studying fair PCA. Prior works define fairness with respect to users by analyzing differences in PCA outputs for different user groups, which requires access to group labels. For instance, past works characterize PCA outputs as unfair if they can be used to recover group membership or if approximation errors differ across groups. Instead, our work focuses on mitigating the underlying mechanisms of unfairness in PCA, which we identify as increasing collaboration among items.}
\begin{tabular}[b]{p{4.5cm} c c}
\toprule
\textbf{Fairness Notion} & 
\shortstack{\textbf{Fairness w.r.t.}\\\textbf{Users / Items}} & 
\shortstack{\textbf{Req.}\\\textbf{Labels}} \\
\midrule
Obfuscate group identifiability  \newline\cite{olfat2019convex, mmd-fair-pca} & Users & \checkmark \\
Balance error across groups \newline\cite{fair-pca, fair-pca-multi-criteria, efficient-fair-pca, closed-form-fair-pca} & Users & \checkmark \\
Increase collaboration (ours) & Items & \\
\bottomrule
\end{tabular}
\label{tab:alg-comparison}
\end{table}

There is also a significant difference in the \textit{motivation} of our work compared to existing papers, resulting in a distinction of when the works can be applied.
Specifically, existing works address situations where there is, a priori, an \textit{external constraint} that deems a particular fairness notion necessary, and these fairness notions are \textit{generic} in the sense that they can be defined in a generic machine learning context (e.g., equalize error across groups).

On the other hand, the motivation of our work is to \textit{identify} unfairness issues that arise \textit{specifically} from the PCA algorithm. The issues that we identify are not generic machine learning issues, and hence they would not necessarily be issues that one would be concerned about a priori. Our work helps elucidate the black-box nature of the PCA algorithm and contributes to situations where one does not have a particular fairness notion in mind but would like to understand what issues can arise from this specific algorithm. 

Analogs of this distinction appear in other areas. For example, in prediction, the seminal work of \citet{hardt2016equality} studies how to learn a classifier with an external fairness constraint (equality of opportunity). In contrast, \citet{chen2018my} and \citet{khani2020feature} also study fairness in prediction, but the goal is to identify mechanisms that induce bias in prediction, rather than the goal of developing algorithms that satisfy a fairness notion. 

\paragraph{\textbf{Weighted Dimensionality Reduction}}
Weighted dimensionality reduction is a well-studied problem and several past works have even leveraged these models to mitigate popularity bias \cite{steck2011item}. We note that a common challenge in weighted dimensionality reduction works is that it is difficult to enforce 1) convexity and 2) a hard rank constraint, where the rank of the approximation matrix is strictly upper bounded. The majority of weighted dimensionality reduction works setup non-convex problems and then optimize via gradient descent, alternating least squares, or expectation-maximization \cite{gantner2011personalized, steck2011item, bailey2012principal}. On the other hand, Max-Margin matrix factorization \cite{srebro2004maximum} is an example of a convex formulation; however, this method does not satisfy the hard rank constraint.

\paragraph{Handling missing data.}
Past works in causal inference have also proposed re-weighting methods to mitigate ``popularity bias'' for CF \cite{liang2016causal, schnabel2016recommendations}. These works tackle the problem that ratings are not missing at random and thus need to be re-weighted through techniques such as Inverse Propensity Weighting. In this line of work, the ``popularity bias'' stems from the ``missingness'' of the entries of the matrix.
In contrast, our work does not require that there are missing entries.
We study a different type of issue that is caused by the varying popularity of artists, which arises even if entries are not missing.
\cref{tab:re-weighting-comparison} summarizes the comparison between our work and past work on weighted dimensionality reduction.

\begin{table}
\centering
\Description{Table comparing features of past approaches to re-weighting for dimensionality reduction.}
\caption{Comparison with existing re-weighting techniques for dimensionality reduction, such as Inverse Propensity Weighting (IPS). Unlike prior works, our algorithm, \algname{}, solves a convex optimization problem while also enforcing a hard rank constraint. Further, the motivation for re-weighting in our work is independent of missing data. Instead, it is predicated on the fact that certain items are more popular than others. ``MF'' stands for matrix factorization.}
\begin{tabular}[b]{p{3.2cm} c c c}
\toprule
     & \textbf{Convex} & \textbf{Hard} & \textbf{Beyond}\\ 
    \textbf{Algorithm} & \textbf{Opt.} & \textbf{Rank} & \textbf{Missing Data}\\ \midrule
    \emph{IPS}~\cite{liang2016causal, schnabel2016recommendations} & & \checkmark &\\ 
    \emph{Weighted MF}~\cite{steck2011item, bailey2012principal, gantner2011personalized} & & \checkmark & \checkmark\\ 
    \emph{Max-Margin MF}~\cite{srebro2004maximum} & \checkmark & & \checkmark\\
    \algname{} (ours) & \checkmark & \checkmark & \checkmark\\
\bottomrule
\end{tabular}%
\label{tab:re-weighting-comparison}
\end{table}

\paragraph{\textbf{Popularity Bias}} Our work is adjacent to measuring and mitigating popularity bias in recommender systems. Popularity bias refers to the phenomena in which recommender systems disproportionately recommend already popular items, resulting in low exposure for low-popularity items and reduced discovery for users \cite{klimashevskaia2024survey}. In terms of measuring popularity bias, common metrics include summary statistics such as the coverage of the recommended items over the entire item set or the average popularity of recommended items \cite{abdollahpouri2019managingpopularitybiasrecommender}. More recently, popularity bias has also been defined in terms of equal opportunity, ensuring that head (popular) and tail (less popular) items have balanced true positive rates \cite{zhu2021popularitybias}. In terms of mitigating popularity bias, common approaches include post-processing with popularity compensation \cite{zhu2021popularityopportunity} or regularizing to minimize the correlation between prediction scores and popularity \cite{zhu2021popularityopportunity} or to increase recommendation diversity \cite{abdollahpouri2017controlling}. 

While the majority of popularity bias works focus on the recommendation output or scores, our work focuses on how the representations of popular and less popular items differ and understanding the modeling mechanisms that introduce these differences. Among the works that do consider popularity-aware representations, the focus is often on low data availability for tail items and transferring knowledge from head items \cite{chang2024cluster}; instead, in our work, we focus on ensuring dimensionality reduction methods preserve the preference data that has already been collected for less-popular items.

\section{Unfairness of PCA for Collaborative Filtering (CF)}\label{sec:unfairness_of_pca}

In this section, we describe two mechanisms of PCA that induce item-side unfairness. We first provide a background on PCA (\cref{sec:pca_prelim}), and then we describe the two mechanisms using an empirical example (\cref{sec:unfairness_empirical}). We then show that these mechanisms occur generally on a stylized class of matrices that represent user-item preferences (\cref{sec:unfairness_theory}). Throughout the rest of the paper, we use the notation denoted in \cref{tab:notation}.
\begin{table}[h]
    \centering
    \Description{A notation table.}
    \caption{Notations used in this paper.}
    \begin{tabular}{p{0.165\columnwidth} p{0.75\columnwidth}}
        \toprule
         \textbf{Symbol} &  \textbf{Meaning} \\ \midrule
         $n$, $m$ & number of users and items \\
         $X \in \bR_+^{n \times m}$ & user preference matrix \\
         $\hX$ & approximated user preference matrix \\
         $X_i$ , $X_{.j}$ & $i^{\text{th}}$ row and $j^{\text{th}}$ column of $X$, respectively \\ \hline
         $r$ & number of principal components \\
         $P \in \bR^{m \times m}$ & projection matrix for $X$ \\
         $P_r, I_r$ & the projection and identity matrices with rank $r$ \\
         $P'$ & a projection matrix with diagonal set to zero \\
         $U \in \bR^{m \times r}$ & matrix of $r$ principal components \\ \midrule
         $\sigma_i$ & the $i$-th leading singular value of a matrix \\
         $\{X_n\}$ & a seq. of preference matrices with increasing $n$ \\
         $[m]$ & $\{1, 2, \dots , m\}$ \\ \hline
         $w_j$ & weight for item $j$ in \algname{} \\
         $p_j$ & popularity for item $j$ (see \cref{def:popularity})\\
         $\gamma$ & popularity up-weight-factor in \algname{} \\
         $\mathbf{y}$ & ground-truth binary interaction label \\
         $\epsilon$ & optional accuracy tolerance hyperparameter in \algname{} \\
        \bottomrule
    \end{tabular}
    \label{tab:notation}
\end{table}

\subsection{PCA Preliminaries} \label{sec:pca_prelim}
Let $X \in \mathbb{R}_+^{n \times m}$ be a non-negative matrix of preferences over $n$ users and $m$ items. The entries of $X$ can be explicit score ratings for items or implicit interaction counts (e.g. clicks).
PCA applied to $X$ projects the matrix into a $r$-dimensional space, producing an approximation matrix $\widehat{X}$, where $r \ll m$ is a user-defined rank hyperparameter. Formally, PCA solves:
\begin{align} \label{eqn:vanilla-pca-objective}
\begin{split}
    \argmin_{P = UU^T}  &\| X - XP \|_F^2 \\
    \text{s.t.} \quad & U \in \mathbb{R}^{m \times r}, U^TU = I_r
\end{split}
\end{align}
The optimization is over projection matrices $P = UU^T$ where the columns of $U \in \mathbb{R}^{m \times r}$ form an orthonormal basis. The optimal projection matrix $P^*$ minimizes the reconstruction error $\|X - \widehat{X}\|_F^2$ between the original matrix and the approximation, where $\hX = XP^*$.

Note that the approximation matrix $\hX = XP^*$ is equivalent to taking the $r$-truncated Singular Value Decomposition (SVD) of $X$. 

The entries of $P$ can be interpreted as item-item similarity values, highlighting the collaborative nature. For instance, the prediction score for user $i$ and item $j$ is:
\begin{equation}\label{eqn:prediction-linear-combination}
    \hX_{ij} = \sum_{j'=1}^m P_{jj'} X_{ij'}
\end{equation}
That is, the predicted score for item $j$ is the linear combination of user $i$'s score for every other item weighted by the similarity with item $j$. Intuitively, if items $j$ and $j'$ are similar the score should also be similar. The goal of PCA then is to best approximate $X$ by learning a low-rank item-item similarity matrix $P$.

\paragraph{Recommendation Setup.}
We consider the recommendation setting in which $X \in \bR^{n \times m}$ is a matrix of preferences of $n$ users across $m$ items. Given that PCA optimizes the reconstruction error over the entire matrix, we assume missing values are imputed with a default value; in all of our experiments, the default value is $0$. An algorithm, such as PCA, then learns a low-rank approximation $\hX$, and the items that user $i$ has not previously rated or interacted with but have the highest values in the row $\hX_i$ are recommended. This setting is modeled from the collaborative filtering work of \citet{sarwar2001item}. In our analysis, we focus first on whether the projection matrices PCA yields reflect item-item collaboration before assessing the downstream performance of the recommendations.

\subsection{Two Unfairness Mechanisms of PCA}  
\label{sec:unfairness_empirical}

We describe two mechanisms of PCA that induces unfairness at the item level, and we demonstrate that these mechanisms arise on an empirical example. We use the lastfm-2k dataset \cite{lastfm}, which records the number of times a user of the LastFM music platform listened to their favorite artists. 
Specifically, if artist $j$ is one of user $i$'s top 25 artists, then $X_{ij}$ is the number of times user $i$ listened to artist $j$. Otherwise $X_{ij} = 0$.  
We use a dataset with $n = 1,867$ users and $m=1,529$ artists. To account for heterogeneity in user listening volume we pre-process by row-normalizing the matrix before applying PCA. See the Experiments section for a detailed description of this dataset.
We compute PCA on this matrix $X$ for all possible values of the rank $r$, from $0$ to $m$.
Let $P_r \in \bR^{m \times m}$ be the projection matrix corresponding to the output of PCA for rank $r$. Now we describe two ways in which PCA induces unfairness for items (artists).

\subsubsection{Mechanism 1: Unfairness for unpopular items.}
The overall reconstruction error, $\|X - XP_r\|^2_F$, decreases as $r$ increases in a diminishing returns fashion: see the dashed gray line in \cref{fig:teaser}, Subfigure A.
In fact, it can be shown that the reconstruction error decreases by exactly $\sigma_r^2$ at rank $r$ compared to $r-1$, where $|\sigma_1| \geq \dots \geq |\sigma_m|$ are the ordered singular values of $X$ (see \cref{thm:diminishing-returns} in the Appendix).

However, this pattern of diminishing returns does not occur at the individual item level. 
We define the \textit{normalized item error} for item $j$ as $\|X_{.j} - \hX_{.j}\|^2_2 / \|X_{.j}\|_2^2$, where $X_{.j}$ and $\hX_{.j}$ are the $j$'th columns of $X$ and $\hX$, respectively.
Subfigure A in \cref{fig:teaser} plots the normalized item error for eight individual artists (items), displaying the large heterogeneity in how errors decrease as a function of rank.  
For each artist, the normalized error is initially 1 when the rank is $0$ since $P = 0$, and drops sharply after some threshold rank is reached, where this threshold varies greatly by the artist.
For instance, the artist Claudia Leitte requires the rank to be greater than $1000$ before their normalized error decreases below 20\%.

In general, the leading components of PCA capture the artists who are popular. Let us define the popularity $p$ of item $j$ be the norm of $X_{.j}$, as shown in \cref{def:popularity}.
\begin{definition}\label{def:popularity}
    The popularity $p_j$ of item $j$ is $p_j = \|X_{.j}\|_2$.
\end{definition}
Subfigure B in \cref{fig:teaser} shows the relationship between artist popularity and the number of principal components needed to halve the initial reconstruction error of $\| X_{.j} \|_2^2$. The subfigure shows that leading principal components greatly reduce reconstruction errors for popular artists. The top-$10\%$ most popular artists require $78$ components, on average, to halve their error while the bottom $10\%$ requires $455$ of $1,529$ components.

\subsubsection{Mechanism 2: Unfairness for popular items.}
We now describe a completely different mechanism that negatively impacts popular items.
The previous mechanism showed that the leading components favor the popular items.
However, we find that the leading components can become \textit{specialized} in \textit{individual} items, which has undesirable consequences in the context of collaborative filtering.

Recall that PCA outputs a projection matrix $P \in \bR^{m \times m}$. 
We claim that it is undesirable for item $j$ for the diagonal entry, $P_{jj}$, to be close to 1 at low values of $r$, which is the case for popular artists as seen in Subfigure C of \cref{fig:teaser}.

Per \cref{eqn:prediction-linear-combination}, for an artist $j$, the approximation of the listening count for user $i$ is a linear combination of user $i$'s listening counts for all other items weighted by the item similarities. Now, if it is the case that the diagonal entry $P_{jj} = 1$, and  $P_{jj'} = 0$ for all $j' \neq j$, we recover a perfect reconstruction ($\hX_{ij} = X_{ij}$).
But this implies that the reconstructed preference of item $j$ is simply the original preference towards item $j$, which is not useful information in the context of collaborative filtering for recommendations.
This does not give us a way to infer whether a user will like item $j$ given their preferences over other items as all previously unseen artists will receive the same default prediction score.
The diagonal entry $P_{jj}$ being close to 1 implies that most of the reconstructed value for $\hX_{ij}$ is coming from $X_{ij}$.

\subsection{Theoretical Guarantees on a Stylized Class of Matrices}
\label{sec:unfairness_theory}
We now demonstrate that PCA asymptotically exhibits the above two mechanisms---to an extreme---in a class of matrices that represent user-item preferences with popularity heterogeneity.

We consider a sequence of matrices of increasing size, where both the number of users and items is growing.
Concretely, consider a sequence of matrices $\{X_n\}_{n \geq 1}$, where $X_n \in \{0, 1\}^{n \times m_n}$ and $m_n = o(n)$.
The $(i, j)$'th entry of $X_n$ is 1 if user $i$ interacts with item $j$, and 0 otherwise.

We assume that the items can be partitioned into two classes: popular items and unpopular items.
We assume that the first $M_n \leq m_n$ items ($j \in [M_n]$) are the popular items for $X_n$ and satisfy the following assumption:
\begin{assumption}[Popular items] \label{assump:popular}
Let $X'_n \in \{0, 1\}^{n \times m_n}$ be a copy of $X_n$ where all entries in columns $j > M_n$ are set to zero.
Then, we assume that the $M_n$'th largest singular value of $X'_n$, which we denote by $\sigma_{M_n}(X'_n)$, grows as $\Omega(\sqrt{n})$.
\end{assumption}
\begin{remark}
\cref{assump:popular} is satisfied with high probability if all entries of $X'_n$ are i.i.d. mean zero subgaussian random variables with unit variance; see Theorem 1.1 in \citet{rudelson2009smallest} and \cref{fig:bernoulli_eigvals} in the Appendix for empirical validation.
\end{remark} 

Next, we assume that for unpopular items, the number of users that like the item is a constant.
\begin{assumption}[Unpopular items] \label{assump:unpopular}
The popularity of all unpopular items is upper bounded by a constant $K$.
That is, for all $n$, $p_j \leq \sqrt{K}$ for any $j > M_n$.
\end{assumption}

Then, we show that PCA on the matrix $X_n$ using the top $M_n$ principal components admits the two undesirable mechanisms.
Let $I_{n, M_n} \in \bR^{m(n) \times m(n)}$ be the matrix where all entries are zero except for the first $M_n$ diagonal entries, which are 1. The proof for \cref{thm:unfair} is in \cref{sec:existence-thm-proof}.
\begin{theorem} \label{thm:unfair}
Let $P_n \in \bR^{m_n \times m_n}$ be the projection matrix given by performing PCA on matrix $X_n$, taking the largest $M_n$ principal components.
Then, $||P_n - I_{n, M_n}||_F \to 0$ as $n \to \infty$.
\end{theorem}

\cref{thm:unfair} states that as the number of users increases, the projection matrix outputted by PCA with $M_n$ components converges to the $I_{n, M_n}$ matrix.
The projection matrix being $P = I_{n, M_n}$ demonstrates both undesirable mechanisms.
The proof makes use of the Davis-Kahan theorem from perturbation theory, which can be found in the Appendix.

Firstly, all columns $j > M_n$ that represent the unpopular items are the 0 vector in the projection matrix; i.e. the projection does not contain \textit{any} information about item $j$.
Then, the reconstruction, $\hat{X}_{.j}$ will also be the 0 vector; that is, the reconstructed preference of every user to every unpopular item is outputted to be 0.

Next, fix a popular item $j \leq M_n$.
Then, column $j$ of the projection matrix approaches $e_j$, the unit vector with 1 in the $j$'th entry.
Then, the reconstruction of the preference of user $i$ for item $j$, $\hat{X}_{ij}$, is exactly $X_{ij}$. 
That is, the reconstruction for the $(i, j)$'th entry simply ``reads'' the value that was there in the original matrix.
This provides a perfect reconstruction, but this provides no useful information in the context of collaborative filtering.
The reconstruction only provides non-zero values to entries that already existed in the original matrix, which does not serve the purpose of using this method as a recommendation tool.
A projection matrix that is useful for recommendations should contain many non-zero entries for column $j$: then, the preference of user $i$ towards item $j$ can be inferred through the existing preferences of user $i$ towards \textit{other} items $j' \neq j$.
\section{Our Proposed Algorithm: Item-Weighted PCA}\label{sec:algorithm}
We propose an algorithm dubbed \algname{} that counters the unfairness mechanisms introduced in the previous section by up-weighting less popular items. 
We first formally state the problem we aim to solve, and then present \algname{} as an algorithm that efficiently and exactly solves the problem.
We show that PCA is an instance of \algname{}, as is Normalized PCA (a common re-weighting baseline) for a subset of preference matrices. Moreover, \algname{} offers a way to interpolate between the two baselines. On a class of stylized matrices, we show there exist input preference matrices $X$ such that interpolation is needed.

%%%%%%%%%%%%%%%%%%%%%%%%%%%%%%%%%%%%%%%%%%%%%%%%%%%%%%%%%%%%%%%%%%%%%%%%%%%%%%%%%%%%%%%%%%%%%%%%%%%%%%%%%%%%%%
\subsection{Algorithm Description}
\subsubsection{Problem Statement.}
The unfairness mechanisms in \cref{sec:unfairness_of_pca} show that specialization on popular items ensures that preferences for highly popular items are preserved in $\hX$ while preferences for low popularity items are lost. As such, to mitigate popularity-specialization, and in turn increase collaboration, we aim to ensure preferences for low popularity items are reflected in $\hX$ even in the low-rank regime. 

Formally, let $X \in \mathbb{R}^{n \times m}$ be an input matrix, where entries denote preference scores and missing values are set to zero, and
$r \ll \min\{n, m\}$ be a rank parameter.
Let $w_j \geq 0$ for $j \in [m]$ be item-specific weights that are input parameters.
We aim to solve the following problem to uncover principal components $U$ that account for low-popularity items:
\begin{align} 
    \argmax_{U\in \mathbb{R}^{m \times r}}  \quad&\sum_{j=1}^m w_j \;
    \langle  X_{.j}, \hX_{.j} \rangle \label{eqn:fair-pca-obj}\\
    & \text{s.t.} \quad U^TU = I  \label{eqn:fair-pca-constraints}
\end{align}
where $\hX = XUU^T$.

The above problem is posed in response to the unfairness mechanisms identified in PCA. Suppose user $i$ has a high preference score for item $j$. While in PCA, the error $\left(X_{ij} - \hX_{ij}\right)^2$ is only one of $n^2$ terms in the global reconstruction error, we now aim to maximize $\left(w_j X_{ij}\right) \hX_{ij}$, which incentivizes a large prediction score $\hX_{ij}$. Intuitively, up-weighting the prediction scores for low-popularity items induces collaboration because if $P$ is low-rank, the projection matrix cannot ``afford'' to dedicate one component to a single item. 

Note that the weights $w_j$ must be given as input. To ensure the weights uplift less popular items, we define the weight for item $j$ in terms of its popularity $p$:
\begin{equation}\label{eqn:weight-def}
    w_j = p_j^{\gamma}
\end{equation}

In all of our experiments, we set $\gamma=-1$, where we show in \cref{sec:theory_item_weights} that setting $\gamma=-1$ serves as an interpolation between two PCA baselines. 

%%%%%%%%%%%%%%%%%%%%%%%%%%%%%%%%%%%%%%%%%%%%%%%%%%%%%%%%%%%%%%%%%%%%%%%%%%%%%%%%%%%%%%%%%%%%%%%%%%%%%%%%%%%%%%
\subsubsection{Algorithm.}\label{sec:alg-subsection}
Our goal is to design an algorithm that efficiently solves the item-re-weighting problem in \eqref{eqn:fair-pca-obj}-\eqref{eqn:fair-pca-constraints}. Specifically, our desiderata are that the algorithm solves a convex optimization problem and runs in polynomial time. As discussed in \cref{sec:related_work}, it is difficult to perform weighted matrix factorization while enforcing convexity. For instance, re-weighting the PCA loss function as $L = \sum_{j=1}^m w_j \|X_{.j} - \hX_{.j}\|^2$ violates convexity (see \cref{sec:trivial-re-weighting} for a precise claim and proof).

As such, we propose the algorithm \algname{}, which solves \eqref{eqn:fair-pca-obj}-\eqref{eqn:fair-pca-constraints} while respecting both optimization desiderata. Our approach begins by relaxing the feasible set. Instead of constraining to projection matrices $P=UU^T$, \algname{} relaxes to optimize over positive semi-definite matrices (PSD) with bounded trace and eigenvalues and solves for an extreme-point optimal solution to the following Semi-Definite Program (SDP):

\begin{align} 
    \begin{split}
        \argmax_{P\in\bR^{m \times m}}  \quad&\sum_{j=1}^m w_j \; 
        \langle  X_{.j}, \hX_{.j} \rangle\\
        \text{s.t.} \quad  
        & \text{tr}\left(P\right) \leq r, 0 \preceq P \preceq 1 
    \end{split}\tag{\algname{}}
\end{align}

Observe that the set of PSD matrices with trace $\leq r$ and eigenvalues $\in [0, 1]$ is a superset of rank $r$ projection matrices. In \cref{thm:comp_efficient}, we claim that the PSD convex relaxation is tight in that the solution to \algname{} is also the solution to the item re-weighting problem, further \algname{} yields a polynomial-time algorithm for the optimization problem of \eqref{eqn:fair-pca-obj}-\eqref{eqn:fair-pca-constraints}. The proof is in \cref{sec:efficiency-thm-proof}.

\begin{theorem} \label{thm:comp_efficient}
\algname{} solves a convex optimization problem and is a polynomial-time algorithm to solve the optimization problem of \eqref{eqn:fair-pca-obj}-\eqref{eqn:fair-pca-constraints}.
\end{theorem}

\algname{}'s objective \eqref{eqn:fair-pca-obj} does not aim to reconstruct the original matrix $X$.
However, it is possible to add constraints to enforce a small error if desired.
Suppose $E_r = ||\hX^{\text{PCA}} - X||_F^2$ is the reconstruction error of the PCA solution (which is the smallest possible reconstruction error).
Then, one can add a constraint to the optimization \eqref{eqn:fair-pca-obj}-\eqref{eqn:fair-pca-constraints} of the form 
$||\hX - X||_F^2 \leq (1+\epsilon) E_r$
for some parameter $\epsilon > 0$, so that the reconstruction error of the output $\hX$ is at most a $(1+\epsilon)$ factor of $E_r$.

In \cref{sec:efficiency-thm-proof}, we also show that \cref{thm:comp_efficient} holds with the added constraint.

%%%%%%%%%%%%%%%%%%%%%%%%%%%%%%%%%%%%%%%%%%%%%%%%%%%%%%%%%%%%%%%%%%%%%%%%%%%%%%%%%%%%%%%%%%%%%%%%%%%%%%%%%%%%%%
\subsubsection{Theoretical Comparison with Baselines}\label{sec:theory_item_weights}
We compare \algname{} against PCA and Normalized PCA, two baselines. Henceforth we will use ``Vanilla PCA'' interchangeably with PCA to distinguish from \algname{}. Normalized PCA is a trivial baseline mitigating popularity heterogeneity, involving column normalization before applying PCA. Specifically, define $X' = XD_X^{-1}$, where the diagonal values of $D_X$ are the Euclidean norms of the columns of $X$. Then let $\hX'$ be the Vanilla PCA approximation of $X'$ and utilize $\hX'$ for recommendations.

Now, we contextualize the flexible re-weighting \algname{} offers by showing that Vanilla PCA is a special case of \algname{}, as is Normalized PCA in the case of item-regular block-diagonal matrices, which are defined below. Further, there exists a class of matrices where interpolating between Vanilla and Normalized PCA is necessary.

We define an item-regular block-diagonal preference matrix $X$ as a matrix where items can be partitioned into $K$ mutually exclusive sets. Each user interacts only with items in one set, and all items in set $k$ have popularity $p_k$. We use this class of matrices to contextualize \algname{} in \cref{thm:instantiations}, which is proven in \cref{sec:instantiations-thm-proof}. 

\begin{theorem}\label{thm:instantiations}
    For all preference matrices $X \in \bR_+^{n \times m}$, Vanilla PCA is equivalent to \algname{} with $\gamma=0$. For item-regular block-diagonal matrices, Normalized PCA is equivalent to \algname{} with $\gamma=-2$.
\end{theorem}

From \cref{thm:instantiations}, we can view setting $\gamma=-1$ as an interpolation between Vanilla and Normalized PCA, where low-popularity items are up-weighted compared to PCA but not up-weighted as significantly as Normalized PCA. In \cref{sec:optimality-over-baselines}, we show interpolation is important because beyond a given popularity gap, Normalized PCA overcompensates and overweights unpopular items. 

We emphasize that while we utilize binary interaction matrices to theoretically contextualize \algname{} in this section, \algname{} is applicable to any input preference matrix $X \in \bR_+^{n \times m}$.
\section{Experiments} \label{s:experiments}
We demonstrate that on two real-world recommender systems datasets, \algname{} reduces specialization and increases collaboration while also improving downstream recommendation performance for users. In \cref{sec:dataset-description} we summarize the datasets and pre-processing methodology used; in \cref{sec:eval-method} we introduce two collaboration evaluation metrics; and in \cref{sec:results} we present our evaluation results comparing \algname{} against the various PCA baselines.

%%%%%%%%%%%%%%%%%%%%%%%%%%%%%%%%%%%%%%%%%%%%%%%%%%%%%%%%%%%%%%%%%%%%%%%%%%%%%%%%%%
\subsection{Datasets}\label{sec:dataset-description}
We utilize the lastfm-2k and movielens-1M datasets which are both common recommender systems evaluation datasets. Below, we summarize the contents of the preference matrices $X$ as well as our pre-processing. The dataset summary statistics follow pre-processing are reported in \cref{tab:datasets}

\underline{\textit{LastFM}}: We use the lastfm-2k dataset \cite{lastfm} of user listening counts. 
$X_{ij}$ is set to the number of times user $i$ listened to artist $j$ if artist $j$ is one of user $i$'s top-50 most-listened artists. Otherwise, $X_{ij} = 0$. We filter the dataset to keep only 
users with at least $10$ listening counts and then artists with at least $10$ top listeners among the remaining users, leaving a $1,867 \times 1,529$ matrix. We row normalize the listening counts for all users by dividing each row by its sum. The normalization accounts for the fact that users differ in how much they utilize the LastFM platform.

\underline{\textit{MovieLens}}: We use the MovieLens-1M dataset in which users provide ratings for movies on a scale from $1$-$5$ \cite{movielens}. We uniformly randomly sample $1200$ artists from the original dataset and also define $X$ as an interaction matrix such that $X_{ij} = 1$ if user $i$ left a rating for movie $j$. 

\begin{table}[ht]
\centering
\Description{Table summarizing the LastFM and MovieLens datasets used in our experiments.}
\caption{Summary statistics of datasets after pre-processing.}
\begin{tabular}{l c c c c }
\toprule
\textbf{Dataset} & \textbf{\# Users} & \textbf{\# Items} & \textbf{\# Entries} & \textbf{Entry Type} \\ \midrule
LastFM & 1,867 & 1,529 & 62,975 & (0, 1) \\ 
MovieLens & 6,040 & 1,200 & 346,918 & \{0, 1\} \\ 
\bottomrule
\end{tabular}
\label{tab:datasets}
\end{table}

%%%%%%%%%%%%%%%%%%%%%%%%%%%%%%%%%%%%%%%%%%%%%%%%%%%%%%%%%%%%%%%%%%%%%%%%%%%%%%%%%%
\begin{figure*}[ht]
     \centering
     \begin{subfigure}[b]{1.0\linewidth}
         \centering
         \includegraphics[width=\textwidth]{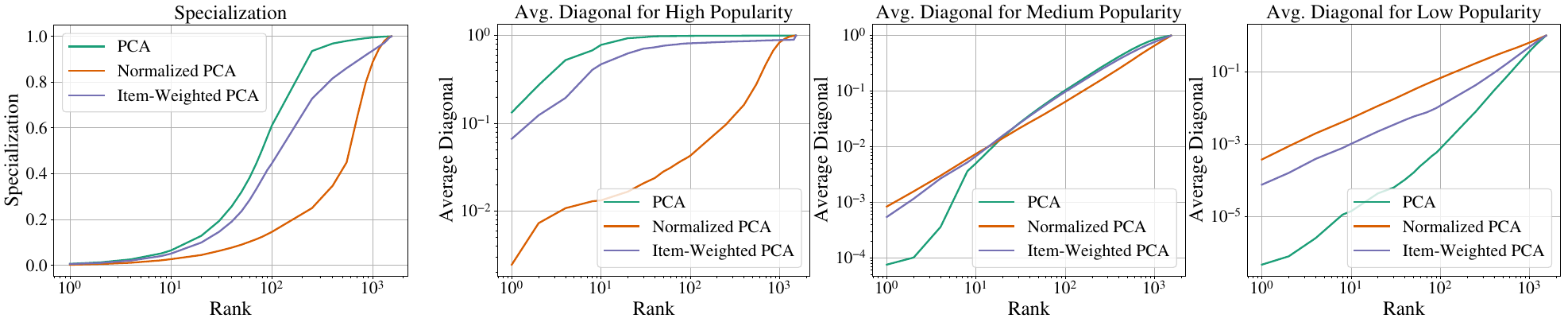}
         \caption{LastFM}
         \label{fig:lastfm-diag}
     \end{subfigure}
     \begin{subfigure}[b]{1.0\linewidth}
         \centering
         \includegraphics[width=\textwidth]{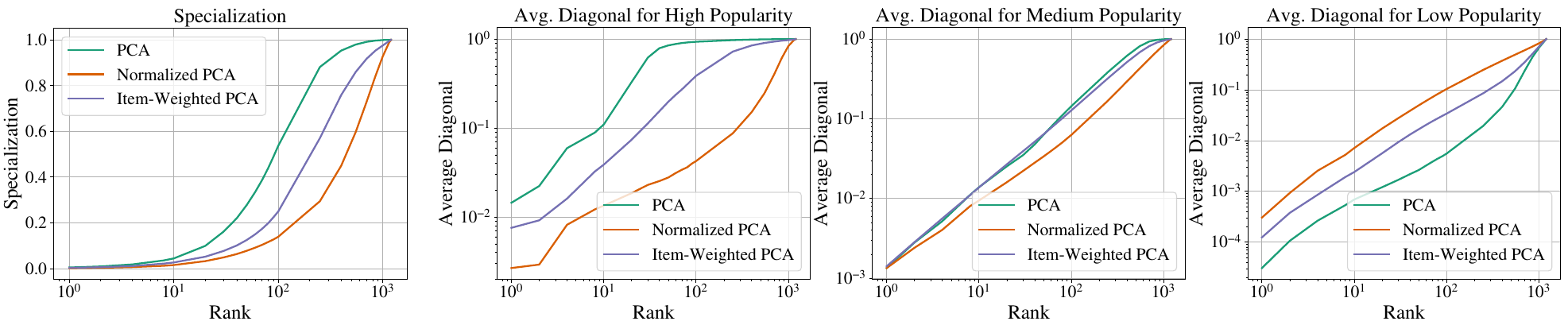}
         \caption{MovieLens}
         \label{fig:movielens-diag}
     \end{subfigure}
        \Description{Plots showing overall specialization and specialization levels for high, medium, and low popularity items in LastFM and Movielens. For each dataset, there are four (overall, high, medium, and low) line plots. Each line plot shows the specialization (y-axis) of PCA, Normalized PCA, and Item-Weighted PCA as a function of rank.}
        \caption{\algname{} reduces the unfairness mechanisms identified in Vanilla PCA in which leading components specialize in individual popular items. The left plots show that \algname{} reduces our specialization metric relative to Vanilla PCA, and the right plots show the diagonal entries for high-popularity items, in particular, decreases. As expected, \algname{} interpolates between Vanilla and Normalized PCA in our in-sample evaluation.}
        \label{fig:item-pref-diag}
\end{figure*}

\subsection{Evaluation Methodology}\label{sec:eval-method}
We compare \algname{} against the baselines of PCA and Normalized PCA. In the remainder of our results section, we call any algorithm that yields a projection matrix for approximating $X$ a ``PCA algorithm'', where Vanilla PCA is the optimal solution to Equation \eqref{eqn:vanilla-pca-objective}. 
Normalized PCA performs column normalization before applying Vanilla PCA.
Throughout we also analyze performance by popularity groups. We define ``high'' popularity as the top $10\%$ most-popular items, ``low'' as the bottom $10\%$, and ``medium'' as all other items. As before, the popularity of an artist $j$ is $\|X_j\|_2$.

\subsubsection{Metrics}
To assess whether \algname{} mitigates the identified unfairness mechanisms -- decreasing specialization and increasing collaboration -- as well as the impact on downstream user recommendations, we introduce three categories of evaluation metrics: \emph{specialization}, \emph{collaboration}, and \emph{downstream recommendation performance}. We describe each of these next.

\paragraph{Specialization.} Building on the unfairness mechanisms introduced in \cref{sec:unfairness_of_pca}, we define specialization as the average of the top-10\% largest diagonal entries of the projection matrix $P$. Though our identified unfairness mechanism focuses on specialization for popular items, in general, specialization can occur for all items so the proposed metric is agnostic to popularity.

\paragraph{Collaboration.} Intuitively, we define high collaboration as the setting where the item-item similarity values in the projection matrix $P$, specifically the off-diagonal values, are useful for recovering user preferences. Specifically, let $P$ be a rank $r$ projection matrix from a PCA algorithm. To evaluate the degree of collaboration exhibited in $P$, let us define $P'$ as $P$ with all diagonal entries set to zero, as diagonal entries correspond to specialization. Then, for an item $j$, the user prediction scores are $\left(XP'\right)_j$. Zeroing out the diagonal forces the prediction scores to be based on the user's ratings for items that are similar to item $j$. 

To evaluate, we perform an \emph{in-sample} analysis and assess the degree to which $XP'$ recovers preferences \emph{already observed} in $X$. Below, we introduce two ranking-based metrics that evaluate whether user-item pairs with interactions seen during training are ranked more highly in $XP'$ than those without interactions.
\begin{equation} \label{eqn:recommender-item-auc-score}\\
    \frac{1}{m} \sum_{j=1}^m \texttt{AUC}\left(XP'_j, \mathbf{y}_j\right) \tag{In-Sample Item AUC-ROC}
\end{equation}
\begin{equation} \label{eqn:recommender-item-precision-score}\\
    \frac{1}{m} \sum_{j=1}^m \texttt{Precision}\left(XP'_j, \mathbf{y}_j, k\right) \tag{In-Sample Item Precision@$k$}
\end{equation}
In the above evaluation metrics, $\mathbf{y}_j$ represents the true binary labels for both datasets, where $\mathbf{y}_j = X_j > 0$. In our experiments, we set $k=20$ for the Precision@$k$ metric. 

\paragraph{Downstream Recommendation Performance.} We utilize standard recommender system methodology and evaluate the quality of user recommendations. We maintain a $20\%$ holdout interaction set and evaluate recommendations by Recall@$k$, Precision@$k$, Normalized Discounted Cumulative Gain (NDCG@$k$), and Mean Reciprocal Rank (MRR@$k$). Our methodology mirrors that of \citet{he2020lightgcn} and in the following results, we set $k=20$.

We emphasize that the collaboration evaluation metrics introduced evaluate the degree to which preferences expressed in the \emph{training} set are preserved in $\hX$, hence the ``In-Sample'' naming. High collaboration does not necessarily imply high downstream user performance on the holdout set.\footnote{We ran all of our experiments in Python on a machine with Intel Xeon E5-2690 CPUs, 2.60 GHz, 30 MB of cache. Our code is available at \href{https://github.com/dliu18/fair-cf}{https://github.com/dliu18/fair-cf}.} 

%%%%%%%%%%%%%%%%%%%%%%%%%%%%%%%%%%%%%%%%%%%%%%%%%%%%%%%%%%%%%%%%%%%%%%%%%%%%%%%%%%
\subsection{Results}\label{sec:results}
We present our results in two phases. First, we present an in-sample analysis of \algname{} and show that our algorithm decreases specialization and increases collaboration compared to Vanilla PCA. Second, we show that \algname{} also improves out-of-sample recommendation performance for users according to standard recommender-system evaluation metrics. Of note, on in-sample collaboration evaluations, \algname{} does not decrease specialization and increase collaboration as much as Normalized PCA, reflecting the interpolation between Vanilla and Normalized PCA. However, we show that \algname{} yields stronger down-stream performance than Normalized PCA when evaluating on out-of-sample data. In \cref{sec:all_baselines}, we also include downstream performance results for more advanced baselines such as weighted matrix-factorization and deep recommender systems; \algname{} yields downstream performance comparable to that of more advanced baselines.
%%%%%%%%%%%%%%%%%%%%%%%%%%%%%%%%%%%%%%%%%%%%%%%%%%%%%%%%%%%%%%%%%%%%%%%%%%%%%%%%%%
\subsubsection{Reduced Specialization}
In \cref{fig:item-pref-diag}, we show that relative to Vanilla PCA, \algname{} decreases specialization; at $r=30$, specialization decreases by $24.3$\% for LastFM and $52.1$\% for MovieLens. Recall that our specialization metric is agnostic to popularity, however, the right plots in \cref{fig:item-pref-diag} confirm that specialization is indeed centered on high-popularity items given that they have the largest diagonal entries, on average. For both datasets, \algname{} decreases the diagonal entries for the high-popularity items.
As expected, Normalized PCA decreases specialization more than \algname{} because low-popularity items receive an even larger up-weight in the objective function.
%%%%%%%%%%%%%%%%%%%%%%%%%%%%%%%%%%%%%%%%%%%%%%%%%%%%%%%%%%%%%%%%%%%%%%%%%%%%%%%%%%
\begin{figure*}[ht]
     \centering
     \begin{subfigure}[b]{1.0\linewidth}
         \centering
         \includegraphics[width=\textwidth]{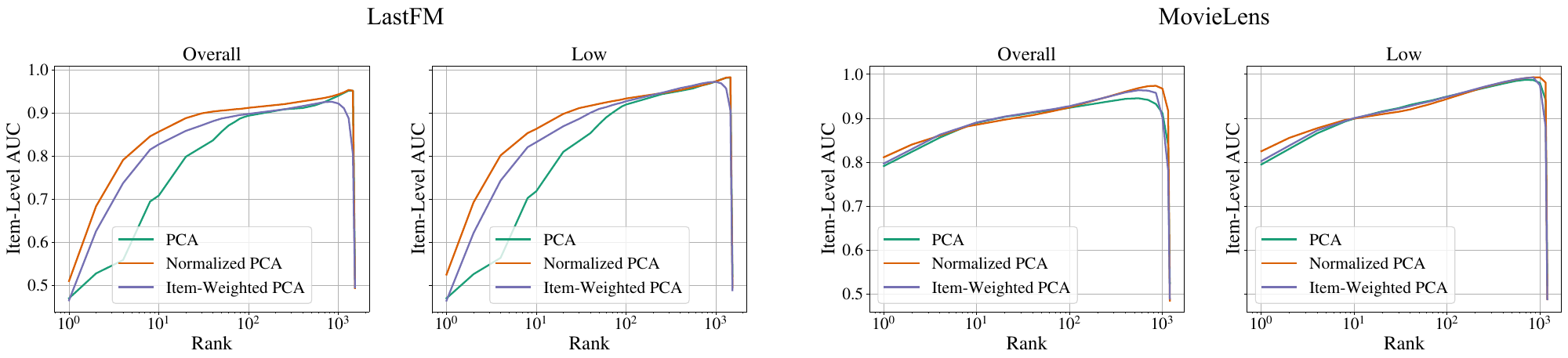}
         \caption{In-Sample Item AUC-ROC}
         \label{fig:in-sample-auc}
     \end{subfigure}
     \hfill
     \begin{subfigure}[b]{1.0\linewidth}
         \centering
         \includegraphics[width=\textwidth]{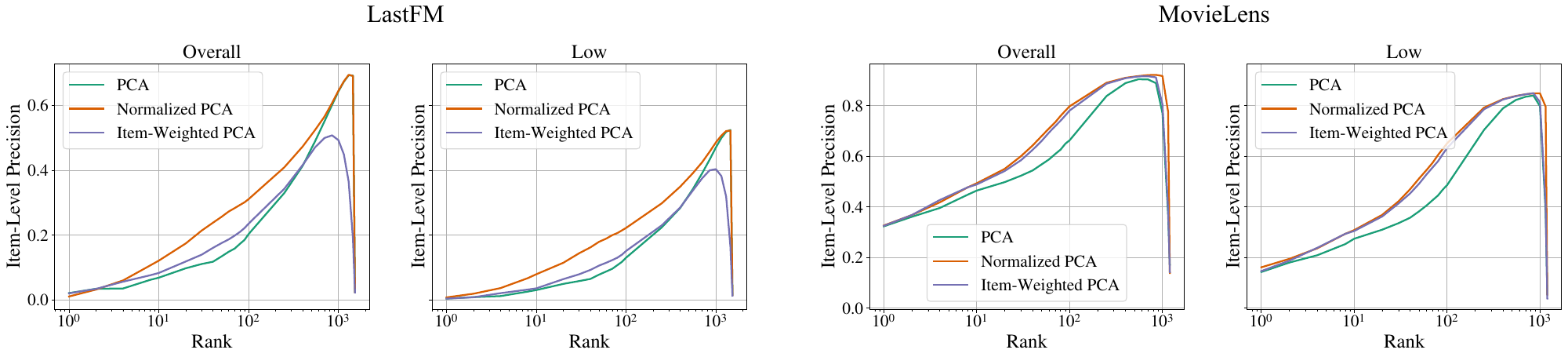}
         \caption{In-Sample Item Precision@$k$}
         \label{fig:in-sample-precision}
     \end{subfigure}
    \Description{Plots showing the amount of collaboration for Item-Weighted PCA and the two baselines. There are two subfigures. The first is collaboration defined with In-Sample Item AUC-ROC and the second is collaboration defined with In-Sample Item Precision@k. For each metric and dataset, there are two line plots, one for overall collaboration and another for collaboration for low-popularity items.}
    \caption{Compared to Vanilla PCA, \algname{} also increases collaboration according to our in-sample evaluation metrics of Item AUC-ROC and Precision@k. The metrics report whether $P$ contains useful item-item similarities for recovering user preferences. As in \cref{fig:item-pref-diag}, the in-sample performance for \algname{} is sandwiched between the two baselines since our method interpolates between the two. In the high-rank regime, collaboration dramatically decreases because the projection matrix can afford to specialize, and off-diagonal entries of $P$, the collaborative entries, approach zero.}
    \label{fig:insample}
\end{figure*}
\subsubsection{Increased Collaboration.}  
\cref{fig:insample} shows that \algname{} increases item collaboration according to our in-sample AUC and Precision@$k$ metrics. For both metrics, we present the average collaboration values (y-axis) across all items (``Overall'') as well as the average values for the low-popularity items (``Low'').  
The curves for all algorithms decrease for large values of $r$ because our evaluation metric zeros out the diagonal, and for large values of $r$ collaborative filtering is not needed as $P$ approaches $I$. Thus, the collaborative value of $P'$ is most relevant in the low-rank regime.

The collaboration improvement for LastFM is most noticeable with the In-Sample Item AUC-ROC metric. For example for $r=30$, the overall In-Sample Item AUC-ROC increases by $6.2$\%, and for low-popularity items, the In-Sample Item-AUC-ROC increases by $6.1$\% with respect to the Vanilla PCA baseline. For MovieLens, the improvement is most noticeable with In-Sample Item Precision@$k$; at $r=30$, the overall In-Sample Item Precision@$k$ increases by $11.7$\% and the In-Sample Item Precision@$k$ increases by $23.1$\% for low-popularity items. While the Normalized PCA baseline has stronger in-sample performance for LastFM, since \algname{} interpolates between the baselines, in MovieLens, \algname{} and Normalized PCA have nearly identical in-sample collaboration performance.
%%%%%%%%%%%%%%%%%%%%%%%%%%%%%%%%%%%%%%%%%%%%%%%%%%%%%%%%%%%%%%%%%%%%%%%%%%%%%%%%%%
\subsubsection{Improved downstream performance for users.} 
\begin{figure}[ht]
     \centering
     \begin{subfigure}[b]{1.0\linewidth}
         \centering
         \includegraphics[width=\textwidth]{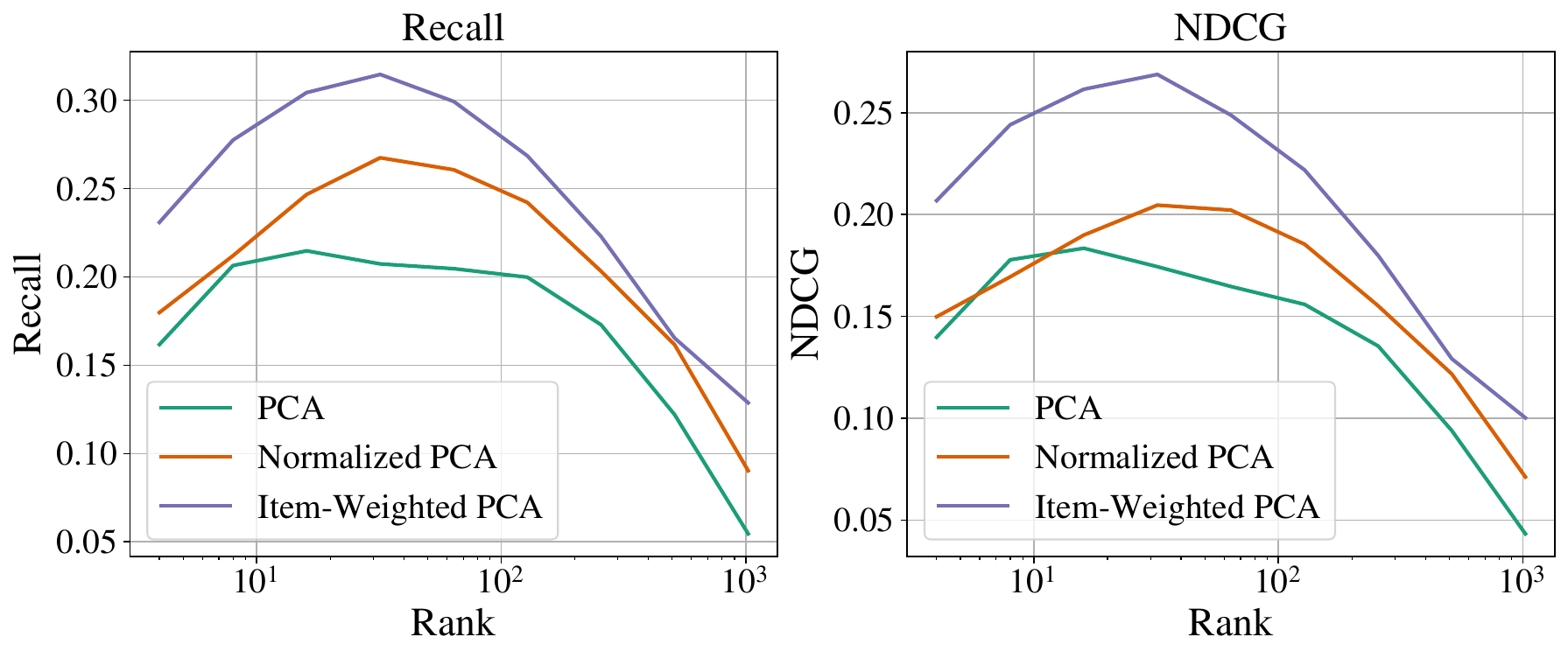}
         \caption{LastFM}
         \label{fig:lastfm_missing_data}
     \end{subfigure}
     \begin{subfigure}[b]{1.0\linewidth}
         \centering
         \includegraphics[width=\textwidth]{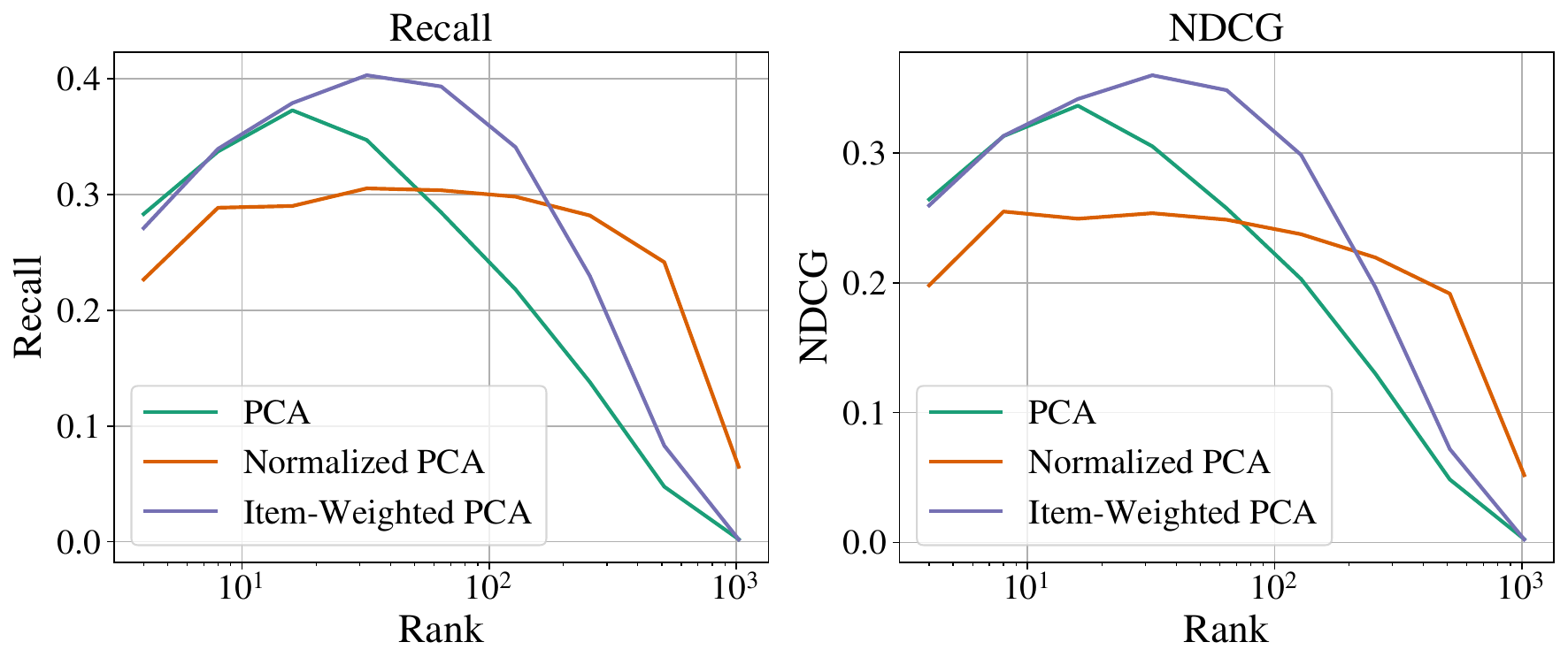}
         \caption{MovieLens}
         \label{fig:movielens_missing_data}
     \end{subfigure}
        \Description{Figure summarizing the downstream performance of Item-Weighted PCA and the two baselines. For each dataset, there are line plots for Recall@k and NDCG@k. In all cases, the x-axis is the rank parameter.}
        \caption{On standard out-of-sample user recommendation metrics, \algname{} improves peak (among all evaluated values of $r$) recommendation performance on Recall@$20$ and NDCG@$20$. While \algname{} demonstrates less in-sample collaboration than Normalized PCA, the downstream performance is stronger, especially in MovieLens. This result demonstrates that \algname{} balances in-sample collaboration with downstream performance. Results for Precision@$k$ and MRR@$k$ are shown in \cref{sec:all_metrics}, and a comparison with additional baselines is included in \cref{sec:all_baselines}.}
        \label{fig:downstream}
\end{figure}
We transition to evaluating the out-of-sample performance of our algorithm to investigate whether the in-sample collaboration improvements translate to downstream benefits. Consistent with past work on recommender systems, we evaluate out-of-sample performance at the user level with the following metrics: Recall@$20$, Precision@$20$, NDCG@$20$, and MRR@$20$. \cref{fig:downstream} shows that \algname{} improves peak user-recommendation performance compared to both PCA baselines on Recall@$20$ and NDCG@$20$. For instance, at $r=32$, the peak in most curves in \cref{fig:downstream}, \algname{} increases Recall@$20$ by $51.7$\% relative to Vanilla PCA for LastFM and $16$\% for MovieLens. The recommendation performance decreases in the high-rank regime because $\hX$ overfits to $X$ and does not generalize to the test set. Precision@$20$ and MRR@$20$ also improve, and these results are included in \cref{sec:all_metrics}.

While the \algname{} performance is sandwiched between the two PCA baselines in our in-sample evaluation, in the out-of-sample evaluation \algname{} outperforms both baselines in the low-rank regime. We conjecture that over up-weighting low-popularity items, while beneficial for collaboration, is detrimental to downstream performance.
\algname{} with $\gamma=-1$ balances in-sample collaboration with out-of-sample performance. 
\section{Discussion}\label{sec:discussion}

\paragraph{Societal Implications of PCA's Unfairness Mechanisms}
The two PCA unfairness mechanisms identified in this work have potential negative societal consequences for both users and items. The first mechanism, identifying that less popular items rely on trailing PCA components, suggests that PCA recommendations do not properly capture the preferences of non-mainstream users. For instance, \cref{fig:teaser} shows that the PCA approximation for the LastFM dataset does not preserve expressed interest in less-popular artists while retaining interest in popular artists. The harm to users who prefer less popular items in PCA echoes a broader trend of poorer recommendations for niche-preferring users~\cite{abdollahpouri2021user}. Likewise, for items, the first unfairness mechanisms suggest that less popular items will have less visibility in recommendations, reinforcing the popularity and perhaps monopoly of already popular items. 

The specialization on individual popular items identified in the second unfairness mechanism indicates that there is an imbalance in representation granularity for popular and less popular items. In cases where a minority group predominantly engages with less-popular content, the dichotomy of granular representations for popular items and coarse representations for less popular items parallels recent works showing that machine learning models can portray subordinate groups as homogenous~\cite{lee2024large} and ``flatten'' minority groups~\cite{wang2025large}. While the previous two works focus on large language models, our work shows that the same pattern can persist in recommendations.

\paragraph{Limitations of \algname{}} We emphasize that \algname{} is designed to mitigate unfairness mechanisms of PCA for recommendations; in other applications of PCA, such as data pre-processing, Vanilla PCA may still be desirable. Furthermore, in our experiments, we set $\gamma=-1$ to interpolate between Vanilla and Normalized PCA. We leave tuning the values of $\gamma \in \mathbb{R}_{<0}$ to future work. In \cref{sec:robustness}, we show that the performance of \algname{} indeed varies within $\gamma\in [-2, 0]$. We also leave for future work more efficient approaches to solving the SDP in \algname{}. In \cref{sec:runtime-complexity}, we show that an upper-bound on the runtime of \algname{} is $\mathcal{O}\left(m^{5.5}\right)$, where $m$ is the number of items.

\section{Conclusions}
By analyzing PCA within the context of collaborative filtering and recommendation systems, we identify two mechanisms of unfairness in PCA. First, information relevant to less popular items is lacking in the leading components. Second, the leading components specialize in individual popular items instead of capturing similarities between items. These mechanisms arise from heterogeneity in item popularities and do not require external group labels to analyze. We illustrate the consequences of these mechanisms in a motivating real-world example and show that these mechanisms provably occur in a stylized setting. To mitigate unfairness, we introduce an algorithm \algname{}, which is designed to preserve user preferences for popular and less popular items. \algname{} is optimal in a stylized setting and our evaluations show that \algname{} not only mitigates the two identified unfairness mechanisms, but also improves downstream performance.

\begin{acks}
    This work started while the authors were visiting the Simons Institute for the Theory of Computing. DL was supported by the National Science Foundation's Graduate Research Fellowships Program, and TER was supported in part by the Aoun Chair Endowment at Northeastern University.
\end{acks}
% \clearpage
\bibliographystyle{ACM-Reference-Format}
\bibliography{references}

\clearpage
\appendix
\section{Proofs} 

\subsection{Proof of Theorem \ref{thm:unfair}} \label{sec:existence-thm-proof}
Let $X'_n \in \{0, 1\}^{n \times m_n}$ be a copy of $X_n$ where all entries in columns $j > M_n$ are set to zero.
We will show that the projection matrix corresponding to performing PCA on $X_n$ is close to the projection matrix of PCA on $X'_n$.

Let $C_n = X_n^{\top} X_n$ and $C_n'= X_n'^{\top} X_n'$.
Let $U_n, U_n' \in \bR^{m_n \times M_n}$ be the matrix whose columns correspond to the $M_n$ normalized eigenvectors corresponding to the $M_n$ largest eigenvalues of $C_n$ and $C_n'$ respectively.

\begin{claim} \label{claim:identity}
$U_n' U_n'^{\top} = I_{n, M_n}$
\end{claim}

\noindent \textit{Proof of \cref{claim:identity}.}
Let $Y_n \in \bR^{M_n \times M_n}$ correspond to the top-left block of $C'_n$.
Let $V_n  \in \bR^{M_n \times M_n}$ have columns that are the eigenvectors of $Y_n$, where $V_n$ is orthonormal (since $Y_n$ is symmetric).
Therefore, $V_n V_n^{\top} = I_{M_n}$ is the identify matrix.
Now, $C'_n$ is simply $Y_n$ in its top left block, and all other entries are 0.
Therefore, if $v \in \bR^{M_n}$ is an eigenvector of $Y_n$, then the vector $v$ padded with zeros, $(v, 0, \dots, 0) \in \bR^{m}$ is an eigenvector of $C'_n$.
Therefore, each column of $U_n'$ is simply an eigenvector $v$ of $Y_n$, padded with 0's to make it a length $m$ vector.
The other eigenvectors of $C_n'$ that do not have this form are the ones whose corresponding eigenvalue is 0, since 0 is an eigenvalue of $C_n'$ with multiplicity $m_n - M_n$.
$\blacksquare$

Now, we use a variant of the Davis-Kahan theorem \citep{davis1970rotation} from \citet{yu2015useful}.
Using the notation in Theorem 2 in \citet{yu2015useful}, we let $r = 1$ and $s = M_n$. Then, using the fact that $||\sin \Theta(U, U')||_F = \frac{1}{\sqrt{2}} ||UU^{\top} - U' U'^{\top}||_F$,
\begin{align}
   || UU^{\top} - I_{n, M_n}||_F \leq \frac{2\sqrt{2}\; ||C - C'||_F} {\lambda_{M_n}(C')},
\end{align}
where $\lambda_{M_n}(C')$ is the $M_n$'th largest eigenvalue of $C'$.
Since all of the less popular items satisfy \cref{assump:unpopular}, every entry in $C - C'$ is upper bounded by $K$.
Therefore, $||C - C'||_F \leq m_n K$.
Next, since the eigenvalues of $C'$ correspond to the square of the singular values of $X'_n$, and since $\sigma_{M_n}(X'_n) = \Omega(\sqrt{n})$, we have that $\lambda_{M_n}(C') = \Omega(n)$.
Therefore, $||P_n - I_{n, M_n}||_F = O(m_n K / n)$, which approaches 0 as $n \to \infty$ since $m = o(n)$ and $K$ is a constant. \qed

\subsection{Proof of Theorem \ref{thm:comp_efficient}} \label{sec:efficiency-thm-proof}
To prove the theorem, we first show that extreme-point optimal solutions to the convex relaxation \algname{} are optimal solutions for the problem statement in Equations \eqref{eqn:fair-pca-obj}-\eqref{eqn:fair-pca-constraints} (the ``original problem''). 

The relaxation in \algname{} is over the feasible set. Instead of optimizing over rank $r$ projection matrices $P = UU^T$, \algname{} optimizes over PSD matrices with bounded eigenvalues and trace. Observe that any optimal solution to the problem posed in the original problem is a feasible solution for \algname{}.
\begin{claim}
    Any optimal solution $P^*$ to the original problem is a feasible solution for \algname{}.
\end{claim}
\begin{proof}
    The optimal solution $P^*$ is a projection matrix that can be factorized as $P^* = U^*U^{*T}$. This factorization is also the eigendecomposition of the matrix where the eigenvalues are $1$ with multiplicity $r$ and $0$ with multiplicity $m-r$. Since the trace of a matrix is the sum of its eigenvalues, the trace constraint $\text{tr}(P^*) = r$ in \algname{} is satisfied. Further since all eigenvalues are $\in [0, 1]$ the eigenvalue constraints $0 \preceq P^* \preceq I_m$ are also satisfied.
\end{proof}
Now, we can prove the theorem if we can show that an extreme-point optimal solution to \algname{} satisfies two properties (i) can be expressed as $UU^T$ where $U^TU = I_r$ and (ii) can be found in polynomial time.

To show (i) we utilize the definition of an extreme point. An extreme point of a convex set is a point that is not a linear combination of two other points in the convex set. For the convex set defined in the constraints of \algname{}, an extreme point must have eigenvalues of $0$ and $1$. 

Suppose there is an extreme point $P' = \sum_{i=1}^m \lambda_i u_iu_i^T$ where there exists a single fractional $\lambda_{i'} \in (0, 1)$. Then it is possible to define $P'$ as the linear combination (average) of the matrices $P' + \epsilon u_{i'}u_{i'}^T$ and $P' -  \epsilon u_{i'}u_{i'}^T$ where $\epsilon \leq \lambda_{i'} \leq 1 - \epsilon$. Note that when there is one fractional eigenvalue, the trace constraint is not tight since $r$ is an integer, thus $P' + \epsilon u_{i'}u_{i'}^T$ is a feasible matrix.

If there are two or more fractional eigenvalues the matrix also cannot be an extreme point. Let $\lambda_1, \lambda_2 \in (0, 1)$. Then we define $P'$ as the average of two matrices: $P' + \epsilon_1\lambda_1u_1u_1^T - \epsilon_2\lambda_2u_2u_2^T$ and $P' - \epsilon_1\lambda_1u_1u_1^T + \epsilon_2\lambda_2u_2u_2^T$ where $\epsilon_1\lambda_1 = \epsilon_2\lambda_2$. Note that the perturbations does not affect the trace of the matrix so the perturbed matrices are feasible even if the trace constraint is tight for $P'$.

Now, since all eigenvalues of extreme points for \algname{} are $0$ or $1$ and the trace is the sum of eigenvalues, the rank of an extreme point is at most $r$ to satisfy the trace constraint. Thus, an extreme point of \algname{} can be eigendecomposed as $UU^T$ where $U^TU = I_r$.

To show (ii) we utilize Theorem 1.8 from \citet{fair-pca-multi-criteria} which states that for SDPs with a linear objective function, $m$ linear (in)equality constraints, and eigenvalue constraints equivalent to those in \algname{}, an extreme-point optimal solution can be found in polynomial time. We elaborate on the runtime complexity of \algname{} in \cref{sec:runtime-complexity}.

Last we discuss the addition of an optional linear reconstruction error constraint and show that \algname{} yields a projection matrix of rank at most $r$. From Theorem 1.8 in \citet{fair-pca-multi-criteria}, we have that all extreme point optimal solutions have rank at most $r$ and can be found in polynomial time. 

To show that an extreme point optimal solution is a projection matrix we must again show that the eigenvalues are integer. We prove by contradiction: consider an extreme point optimal solution $P' = \sum_{i=1}^m \lambda_iu_iu_i^T$ where there exists a single fractional eigenvalue $\lambda_{i'}$. We can show that such a point cannot be optimal because setting $\lambda_{i'} = 1$ would be feasible and improve the objective. The objective function can be written as a linear combination of the eigenvalues of $P'$: $\sum_{i=1}^m c_i\lambda_i$. $c_{i'}$ must be positive, otherwise $P' - \lambda_{i'}u_{i'}u_{i'}^T$ would have a higher objective value. Thus increasing $\lambda_{i'}$ increases the objective. Setting to $1$ also decreases the reconstruction error given that the increase of any eigenvalue decreases reconstruction error. And the perturbed matrix $P' + \left(1-\lambda_{i'}\right)u_{i'}u_{i'}^T$ is feasible for the integer trace constraint given that the trace of $P'$ is at most $r + \lambda_{i'} - 1$.

If there are two fractional eigenvalues, the argument with $\epsilon_1$ and $\epsilon_2$ can again be used to show $P'$ is not an extreme point, where $\epsilon_1$ and $\epsilon_2$ are defined to preserve reconstruction error.
\qed

\subsection{Proofs for \algname{} Analysis}\label{sec:optimality-proofs}

\subsubsection{Theorem \ref{thm:instantiations}}\label{sec:instantiations-thm-proof}
\begin{proof}

From both \citet{olfat2019convex} and \citet{arora2013stochastic}, we can express Vanilla PCA via a tight convex relaxation: 
\begin{align} 
    \begin{split}
        \argmax_{P\in\bR^{m \times m}}  \quad&\sum_{j=1}^m \langle  X_{.j}, \hX_{.j} \rangle\\
        \text{s.t.} \quad  
        & \text{tr}\left(P\right) \leq r, 0 \preceq P \preceq 1 
    \end{split}\tag{Convex Relaxation of PCA}
\end{align}

Note that the above SDP is exactly equal to \algname{} instantiated with $w_j = 1 \forall j \in [m]$ i.e. $\gamma = 0$, proving that Vanilla PCA is equivalent to \algname{} with $\gamma = 0$.

For column normalization, we will show that instantiating \algname{} with $\gamma=-2$ is equivalent to Normalized PCA for item-regular block-diagonal preference matrices. Recall that Normalized PCA solves the following optimization problem to recover a projection matrix:
\begin{align} 
    \begin{split}
        \argmax_{P\in\bR^{m \times m}}  \quad& \|XD - XDP\|^2\\
        \text{s.t.} \quad  
        & \text{tr}\left(P\right) \leq r, 0 \preceq P \preceq 1\\
        & D_{jj} = p_j^{-1}
    \end{split}\tag{Normalized PCA}
\end{align}

Similar to the convex relaxation of PCA, the objective function can be expanded as a trace:
\begin{align*}
    &\|X - XP\|_F^2  = \text{tr}\left(\left(XD  - XDP\right)^T\left(XD - XDP\right)\right)\\
    &= \text{tr}\left(DX^TXD - DX^TXDP - PDX^TXD + PDX^TXDP\right)\\
    &= \text{tr}\left(DX^TXD\right) - \text{tr}\left(DX^TXDP\right) - \text{tr}\left(PDX^TXD\right) + \text{tr}\left(PDX^TDXP\right)\\
    &= \text{tr}\left(DX^TXD\right) - \text{tr}\left(DX^TXDP\right) - \text{tr}\left(DX^TXDP\right) + \text{tr}\left(DX^TXDPP\right)\\
    &= \text{tr}\left(DX^TXD\right) - \text{tr}\left(DX^TXDP\right) - \text{tr}\left(DX^TXDP\right) + \text{tr}\left(DX^TXDP\right)\\ 
    &= \text{tr}\left(DX^TXD\right) - \text{tr}\left(DX^TXDP\right)
\end{align*}
With the last line, the first term is constant, so Normalized PCA minimizes $- \text{tr}\left(DX^TXDP\right)$. In the special case where $X$ is item-regular and block diagonal, the matrix $DX^TXD = D^2X^TX$ given that the covariance matrix $X^TX$ is itself block diagonal and the entries of $D$ are constant within blocks.

Since the diagonal entries of $D$ are $p_j^{-1}$ we have shown that the coefficient matrix $D^2$ in Normalized PCA is equivalent to $\gamma = -2$.
\end{proof}

\subsubsection{Optimality over baselines} \label{sec:optimality-over-baselines}
We provide a concrete instance in which \algname{} balances popular and unpopular items while the PCA baselines yield extreme, undesirable outcomes. Let $X$ be a user preference matrix satisfying the following three assumptions:

\begin{assumption}[Exclusivity] \label{assumption:exclusivity}
    The set of items can be partitioned into $m_p$ popular items and $m_u$ unpopular items. Each user likes either only popular items or only unpopular items.
\end{assumption}

\begin{assumption}[Constant Popularity]\label{assumption:constant}
For all popular items, there are $n_p$ users that like the item, and for all unpopular items, there are $n_u$ users that like the item, where $n_u < n_p$.
\end{assumption}

\begin{assumption}[Exponential Decay] \label{assumption:decay}
    $X_p^TX_p$ and $X_u^TX_u$ are both of rank $r$ and their respective eigenvalues decay exponentially such that for each matrix, the $i^{th}$ largest eigenvalue $\lambda_i = \beta^{-(i - 1)}\lambda_1$, where $\beta > 1$ and $i \leq r$. 
\end{assumption}

Theorem \ref{thm:baseline} states that when the popularity gap is large enough and there are sufficiently many unpopular items, for a rank $r$ projection, Vanilla PCA only reconstructs popular items whereas Normalized PCA only reconstructs unpopular items. \algname{} with $\gamma = -1$, on the other hand, reconstructs both popular and unpopular items in parallel. In the theorem, let $V_p$ be the principal components characterizing the popular items and $V_u$ be the components for the unpopular items.

\begin{theorem}\label{thm:baseline}
    For any binary preference matrix $X$ satisfying Assumptions \ref{assumption:exclusivity}-\ref{assumption:decay}, if $\frac{n_u}{n_p} < \beta^{-2(r - 1)}$ and $m_u = \sqrt{\frac{n_p}{n_u}}m_p$, then the leading $r$ Vanilla PCA components are $V_p$; the leading $r$ column-normalized PCA components are $V_u$. For \algname{} with $\gamma = -1$, half of the leading components are in $V_p$ and the other half is in $V_u$.
\end{theorem}

\begin{proof}
    Observe that the objective for \algname{} can be re-written as:
    \begin{align}
        \sum_{j=1}^m w_j \left<X_{.j}, \widehat{X}_{.j}\right> &= \left<XD, XP\right>\\
        &= \text{Tr}\left(DX^TXP\right) \label{eqn:item-preference-trace}
    \end{align}
    Where $D$ is a diagonal matrix and entry $D_{jj} = w_j$. Thus, the two baselines and \algname{} can be written in terms of Equation \ref{eqn:item-preference-trace} with varying definitions of $D$.

    Observe that the only difference between Equation \ref{eqn:item-preference-trace} and the standard PCA objective is the addition of the weight matrix $D$. Now, we leverage Assumptions \ref{assumption:exclusivity} and \ref{assumption:constant} to show that the weight matrix $D$ does not change the principal components but only their order. To see this, let $V$ be the eigenvectors of $X^TX$. We can write $DV = VD$ because for all entries $i,j$ such that $V_{ij} > 0$, $D_{ii} = D_{jj}$. Thus, the objective for \algname{} becomes:
    \begin{align}
        \text{Tr}\left(DX^TXP\right) &= \text{Tr}\left(D\left(V\Sigma V^T\right)P\right)\\
        &= \text{Tr}\left(\left(V\left(D\Sigma\right)V^T\right)P\right)\\
    \end{align}
    $\Sigma$ is the diagonal matrix of eigenvalues. We can now see that the eigenvectors are still $V$ but the eigenvalues are now scaled to $D\Sigma$. Furthermore, the eigenvectors of $X^TX$ are $\{V_p, V_u\}$ given that $X^TX$ is block diagonal. In the below, let $\lambda_i^u$ be the $i^{th}$ largest eigenvector of $X_u^TX_u$ and $\lambda_i^p$ be the same for $X_p^TX_p$.

    We can bound the sum of eigenvalues of $X_p^TX_p$ as follows:
    \begin{align}
        \sum_{i=1}^r \lambda_i^p &= \text{Tr}\left(X_p^TX_p\right)\\
        &= \|X_p\|_F^2\\
        &= n_pm_p
    \end{align}
    Analogous steps show that the sum of eigenvalues of $X_u^TX_u$ equals $n_um_u$.
    Now, we use Assumption \ref{assumption:decay} to establish the ratio between the leading eigenvalues of the group covariance matrices:
    \begin{align}
        \frac{\lambda_1^p \left(\sum_{i=1}^r \beta^{-i + 1}\right)}{\lambda_1^u \left(\sum_{i=1}^r \beta^{-i + 1}\right)} &= \frac{n_pm_p}{n_um_u}\\
        \frac{\lambda_1^p}{\lambda_1^u} &= \frac{n_p}{n_u} \frac{m_p}{m_u}\\
        \frac{\lambda_1^p}{\lambda_1^u} &= \sqrt{\frac{n_p}{n_u}}\\
    \end{align}
    Thus, all eigenvalues of $X_p^TX_p$ are $\sqrt{\frac{n_p}{n_u}}$ times larger than the corresponding eigenvalue for $X_u^TX_u$.
    
    In the case of Vanilla PCA, $D = I$. We can show that the largest eigenvalue of $X_u^TX_u$ is still smaller than the smallest non-zero eigenvalue of $X_p^TX_p$:
    \begin{align}
        \lambda_r^p &= \lambda_1^p \beta^{1 - r}\\
        &= \lambda_1^u \left(\sqrt{\frac{n_p}{n_u}}\right) \beta^{1 - r}\\
        &\geq \lambda_1^u \beta^{r - 1} \beta^{1 - r}\\
        &\geq \lambda_1^u
    \end{align}
    Thus the $r$ largest eigenvalues correspond with $V_p$. 

    On the other hand, we can show that when $D_{ii} = n_p^{-1} ~\forall i \in I_p$ and $D_{ii} = n_u^{-1} ~\forall i \in I_u$, as in the case of Normalized PCA, the smallest re-scaled eigenvalue for $X_u^TX_u$ will be larger than the largest re-scaled eigenvalue of $X_p^TX_p$
    \begin{align}
        \lambda_r^u & n_u^{-1} = \lambda_1^u \beta^{r - 1} n_u^{-1}\\
        &= \lambda_1^p \sqrt{\frac{n_u}{n_p}}\beta^{1 - r} n_u^{-1}\\
        &= \frac{\lambda_1^p}{n_p} \sqrt{\frac{n_p}{n_u}} \beta^{1 - r}\\
        &\geq \frac{\lambda_1^p}{n_p} \beta^{r - 1} \beta^{1 - r}\\
        &\geq \frac{\lambda_1^p}{n_p}
    \end{align}
    Thus, after re-scaling with column-normalized PCA all of the top $r$ eigenvectors will correspond with $V_u$.

    In the case of \algname{} with $\gamma=-1$, the rescaled $i^{th}$ eigenvalue for $X_p^TX_p$ will be:
    \begin{align}
        \frac{1}{\sqrt{n_p}} \lambda_i^p &= \frac{1}{\sqrt{n_p}} \frac{\sqrt{n_p}}{\sqrt{n_u}} \lambda_i^u\\
        &= \frac{1}{\sqrt{n_u}} \lambda_i^u
    \end{align}
    Thus, the re-scaled eigenvalues for $X_p^TX_p$ exactly equal the rescaled eigenvalues for $X_u^TX_u$. In taking the top $r$ eigenvectors, the final set will contain one half from $V_p$ and another half from $V_u$.
\end{proof}

\subsection{Additional Proofs} \label{sec:additional-proofs}
\subsubsection{Diminishing returns}
\begin{theorem} \label{thm:diminishing-returns}
    Let $X \in \mathbb{R}^{n \times d}$, then the $i^{th}$ principal component reduces the reconstruction error by:
    \begin{equation*}
        \|X - XU_{i-1}U_{i-1}^T\|_F^2  - \|X - XU_{i}U_{i}^T\|_F^2 = \sigma_i^2
    \end{equation*}
    Where the columns of $U_i$ are the leading $i$ principal components and $\sigma_i$ is the $i^{th}$ largest singular value of $X$ by magnitude.
\end{theorem}
\begin{proof}
The reconstruction error $f$ for a given projection matrix $P = UU^T$ can be re-written as:
\begin{align*}
    f(P) &= \|X - XP\|_F^2\\
    & = \text{tr}\left(\left(X  - XP\right)^T\left(X  - XP\right)\right)\\
    &= \text{tr}\left(X^TX - X^TXP - PX^TX + PX^TXP\right)\\
    &= \text{tr}\left(X^TX\right) - \text{tr}\left(X^TXP\right) - \text{tr}\left(PX^T\right) + \text{tr}\left(PX^TXP\right)\\
    &= \text{tr}\left(X^TX\right) - \text{tr}\left(X^TXP\right) - \text{tr}\left(X^TXP\right) + \text{tr}\left(X^TXPP\right)\\
    &= \text{tr}\left(X^TX\right) - \text{tr}\left(X^TXP\right) - \text{tr}\left(X^TXP\right) + \text{tr}\left(X^TXP\right)\\ 
    &= \text{tr}\left(X^TX\right) - \text{tr}\left(X^TXP\right)\\ 
    &= \text{tr}\left(X^TX\right) - \text{tr}\left(U^TX^TXU\right)\\ 
\end{align*}
Vanilla PCA minimizes reconstruction error which is equivalent to maximizing $\text{tr}\left(U^TX^TXU\right)$. The matrix $X^T$ is a symmetric matrix that can be diagonalized as $V\Sigma^2V^T$ where the columns of $V$ are the right singular vectors of $X$ and $\Sigma$ is a diagonal matrix where the diagonal values are the singular values of $X$ sorted by magnitude. 

Maximizing $\text{tr}\left(U^TX^TXU\right)$, where the columns of $U$ are orthonormal then amounts to setting the columns of $U$ to be the leading right singular vectors of $X$. Now the reduction in reconstruction error can be written as:
\begin{align*}
    f\left(P_{i-1}\right) &- f\left(P_i\right) =  \text{tr}\left(U_i^TX^TXU_i\right) - \text{tr}\left(U_{i-1}^TX^TXU_{i-1}\right)\\ 
    &= \text{tr}\left(V_i^T\left(V\Sigma^2V^T\right)V_i\right) - \text{tr}\left(V_{i-1}^T\left(V\Sigma^2V^T\right)V_{i-1}\right)\\
    &= \sum_{j=1}^i \sigma_j^2 - \sum_{j=1}^{i-1} \sigma_j^2\\
    &= \sigma_i^2
\end{align*}
\end{proof}

\subsubsection{Trivial re-weighting} \label{sec:trivial-re-weighting}
We will show that the convex relaxation used for PCA, as introduced in \citet{olfat2019convex} and \citet{arora2013stochastic} is no longer applicable even in the simple case of re-weighting the re-construction error. 

\begin{claim}
    Replacing the PCA objective function with $\sum_{j=1}^m w_j \|X_{.j} - \hX_{.j}\|^2$ results in a non-convex optimization problem.
\end{claim}

\begin{proof}
    If re-formulate the objective in terms of a trace, as in the PCA relaxation, we have the following where the diagonal entries of $D$ are $\sqrt{w_j}$:
    \begin{align}
        &\sum_{j=1}^m w_j \|X_{.j} - \hX_{.j}\|^2 = \|XD - XPD\|^2\\
        & = \text{tr}\left(\left(XD  - XPD\right)^T\left(XD  - XPD\right)\right)\\
        &= \text{tr}\left(DX^TXD - DX^TXPD - PDX^TXD + PDX^TXPD\right)\\
        &= \text{tr}\left(DX^TXD\right) - \text{tr}\left(DX^TXDP\right) - \text{tr}\left(PDX^TXD\right) + \text{tr}\left(PDX^TXPD\right)\\
        &= \text{tr}\left(DX^TXD\right) - 2\text{tr}\left(DX^TXDP\right) + \text{tr}\left(DX^TXPDP\right)
    \end{align}
    Because the right term cannot be simplified with $P^2 = P$, the objective is a quadratic function of the entries of $P$ and in general, non-convex.
\end{proof}

\subsubsection{Runtime complexity of \algname{}} \label{sec:runtime-complexity}

\begin{claim}
    The runtime complexity of \algname{} is $\mathcal{O}^{5.5}$ where $m$ is the number of items.
\end{claim}
\begin{proof}
    The standard form of a semi-definite program over a symmetric matrix $X\in\mathbb{R}^{d \times d}$, consists of (i) a linear objective over $X$ (ii) $c$ linear equality constraints  (iii) the constraint $X \succeq 0$. \citet{jiang2020faster} present an interior point method for solving SDPs with a runtime complexity of:
    \begin{equation} \label{eq:sdp-runtime}
        \mathcal{O}\left(\sqrt{d}\left(cd^2 + c^\omega + d^\omega\right)\log\left(\frac{1}{\delta}\right)\right)
    \end{equation}
Where $\delta$ is a relative accuracy parameter and $\omega$ is the matrix-multiplication runtime exponent.

Let us re-write the SDP in \algname{}, as introduced in \cref{sec:alg-subsection} in standard form. We introduce a change of variables to rewrite $P \preceq I$ as linear equality constraints. Let $Z$ be a $2m \times 2m$ matrix where the upper-left diagonal $m \times m$ submatrix is $P$. Let us introduce linear equality constraints such that the lower-right $m \times m$ submatrix is defined as $I - P$ ($m^2$ element-wise constraints), and the off-diagonal submatrices are defined as $0$ ($2m^2$ element-wise constraints). When reformulating the \algname{} SDP in standard form in terms of $Z$, the constraint $Z \succeq 0$ implies that $P \succeq 0$ and $I - P \succeq 0$, or $I \succeq P$, recovering our desired matrix inequalities. The objective and trace constraint can be trivially re-written in terms of $Z$.\footnote{From \cref{thm:comp_efficient}, the optimal solution to \algname{} is an extreme-point solution, so we can equivalently write $\text{tr}(P) = r$.}

After re-expressing the SDP in terms of $Z$, we have $d = 2m$ and $c=m^2 + 2m^2 + 1 = 3m^2 + 1$. Substituting into \cref{eq:sdp-runtime}, the dominant term is $\mathcal{O}\left(\sqrt{d}c^\omega \log\left(\frac{1}{\delta}\right)\right)$ which can be simplified as $\mathcal{O}\left(m^{2\omega + 1/2} \log\left(\frac{1}{\delta}\right)\right)$. Let us set $\omega \leq 2.5$~\cite{williams2012multiplying} as a conservative upper bound and consider $\delta$ a constant. Then, we have the runtime of \algname{} is $\mathcal{O}\left(m^{5.5}\right)$.
\end{proof}

\section{Supplemental Figures}
\subsection{Singular value scaling of Bernoulli matrices}
We empirically check that \cref{assump:popular} is satisfied for the class of Bernoulli matrices used in \cref{thm:unfair}.

We fix $M = 20$, and vary $n$ from 100 to 100,000.
For each item $j \in [M]$, we draw a random number $p_j \in [0.1, 1]$ uniformly, which represents the Bernoulli parameter for item $j$.
Then, we draw a matrix $X \in \{0, 1\}^{n \times M}$, where $X_{ij} \sim \text{Bernoulli}(p_j)$ independently for each $i, j$.
Then, we denote $s^2_{\min}(X)$ to be the smallest value of the squared singular values of $X$.
For \cref{assump:popular} to be satisfied, $s^2_{\min}(X)$ must scale linearly in $n$.

For one set of values of the Bernoulli parameters $\{p_j\}_{j=1, \dots, M}$, we draw the random matrix $X$ 500 times, and we show the average $s^2_{\min}(X)$ value as well as the 99'th percentile values in Figure~\ref{fig:bernoulli_eigvals}.
The figure shows that the smallest squared singular value does indeed increase linearly with high probability.

\begin{figure}[ht]
    \centering
    \includegraphics[width=0.7\linewidth]{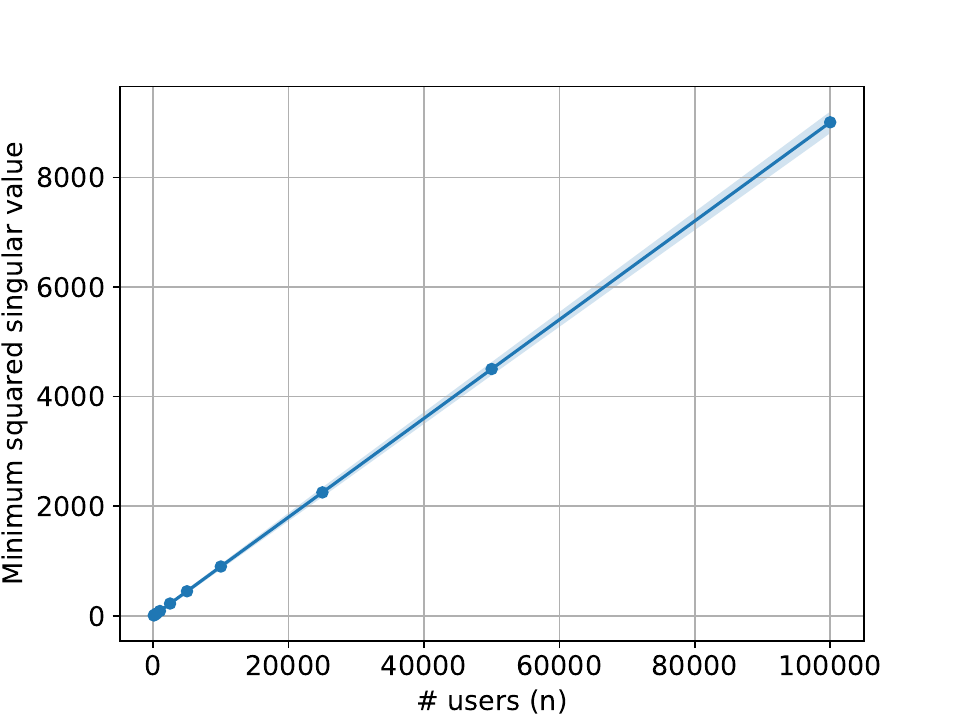}
    \Description{A figure to validate theoretical claims regarding the singular values of Bernoulli matrices. The x-axis is the number of users (rows), and the y-axis is the minimum squared singular value. The points form a line and there is a shaded region, tightly fit around the line, to show the variance of the smallest singular value.}
    \caption{The line corresponds to the average of the smallest squared singular value of the random Bernoulli matrix $X$. The shaded region corresponds to the 1 and 99th percentiles.}
    \label{fig:bernoulli_eigvals}
\end{figure}

\subsection{All performance metrics}\label{sec:all_metrics}
To supplement the Recall@$k$ and NDCG@$k$ metrics shown in \cref{fig:downstream}, \cref{fig:all_metrics} shows that \algname{} increases peak Precision@$k$ and MRR@$k$ as well across both LastFM and MovieLens. For both datasets and all four downstream evaluation metrics, the peak value for \algname{} across all rank budgets is higher than the peaks for the two PCA baselines.
\begin{figure*}[ht]
     \centering
     \begin{subfigure}[b]{0.97\linewidth}
         \centering
         \includegraphics[width=\textwidth]{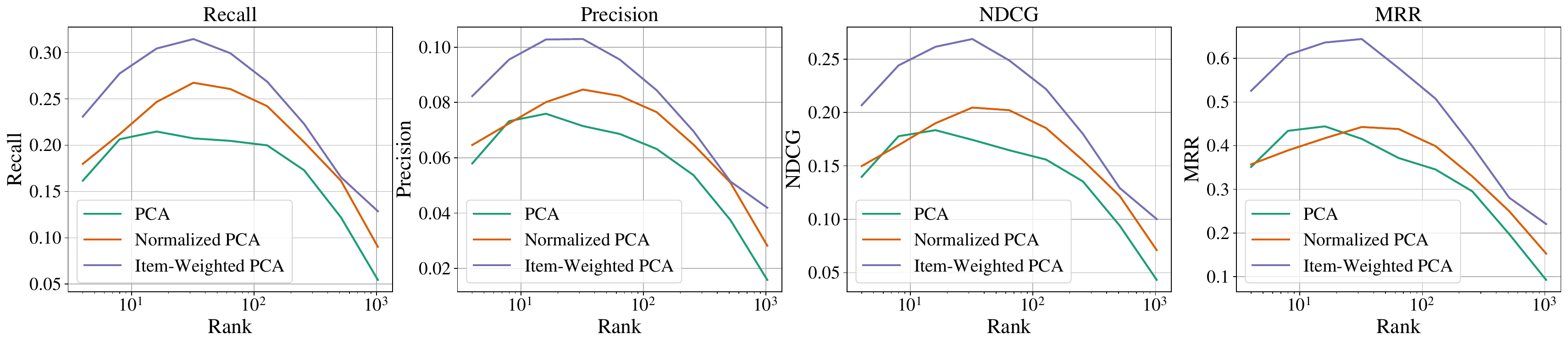}
         \caption{LastFM}
     \end{subfigure}
     \begin{subfigure}[b]{0.97\linewidth}
         \centering
         \includegraphics[width=\textwidth]{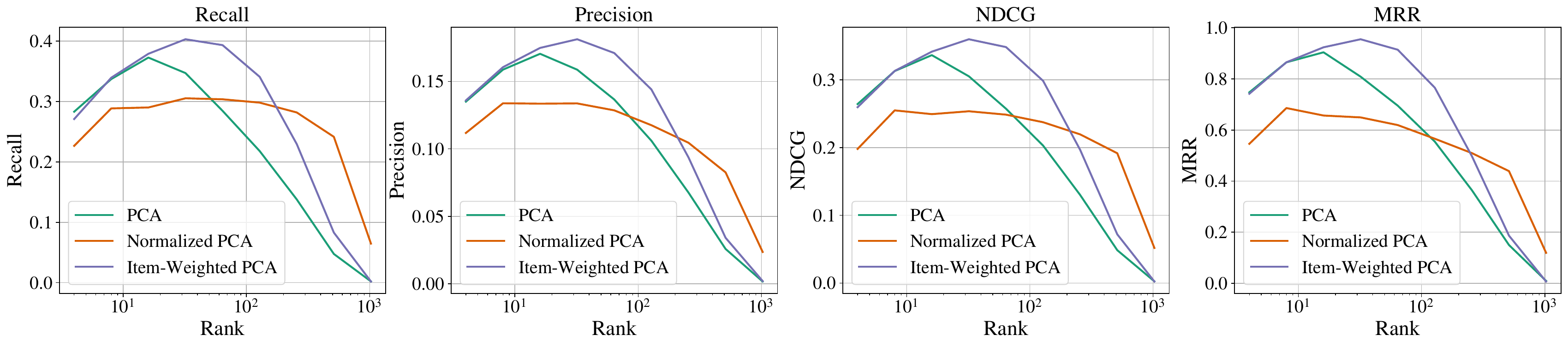}
         \caption{MovieLens}
     \end{subfigure}
    \Description{Performance of PCA models for all four downstream metrics. For each dataset, there are four line plots: Recall@k, Precision@k, NDCG@K, and MRR@K. In all plots, the x-axis represents the rank, and the y-axis is the downstream evaluation metric.}
        \caption{\algname{} increases peak (among all evaluated values of $r$) Precision@$k$ and MRR@$k$ in addition to, as shown in \cref{fig:downstream}, NDCG@$k$ and Recall@$k$ for both LastFM and MovieLens.}
        \label{fig:all_metrics}
\end{figure*}

\subsection{Comparison with broader broader baselines} \label{sec:all_baselines}
\begin{figure*}[h]
     \centering
     \begin{subfigure}[b]{0.97\linewidth}
         \centering
         \includegraphics[width=\textwidth]{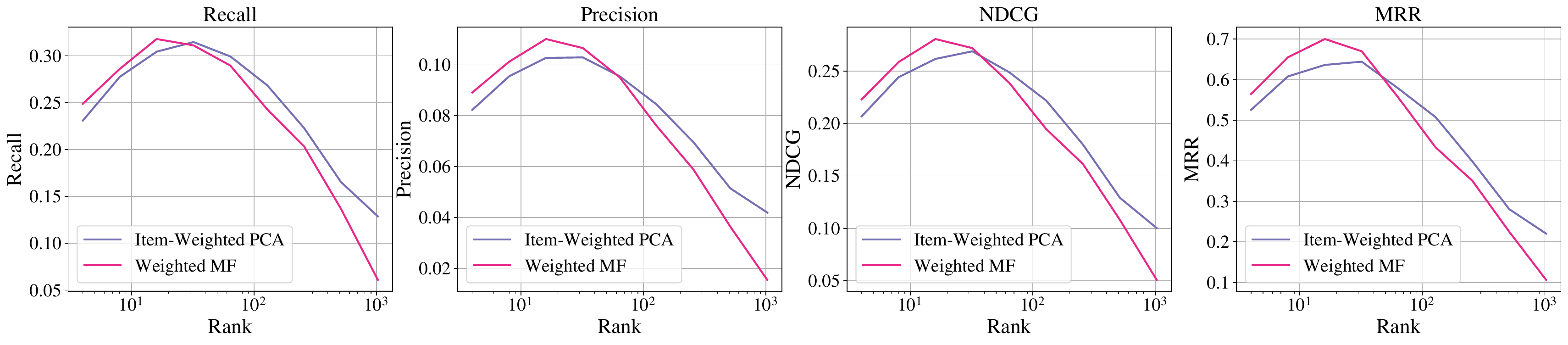}
         \caption{LastFM}
         \label{fig:lastfm_all_baselines}
     \end{subfigure}
     \begin{subfigure}[b]{0.97\linewidth}
         \centering
         \includegraphics[width=\textwidth]{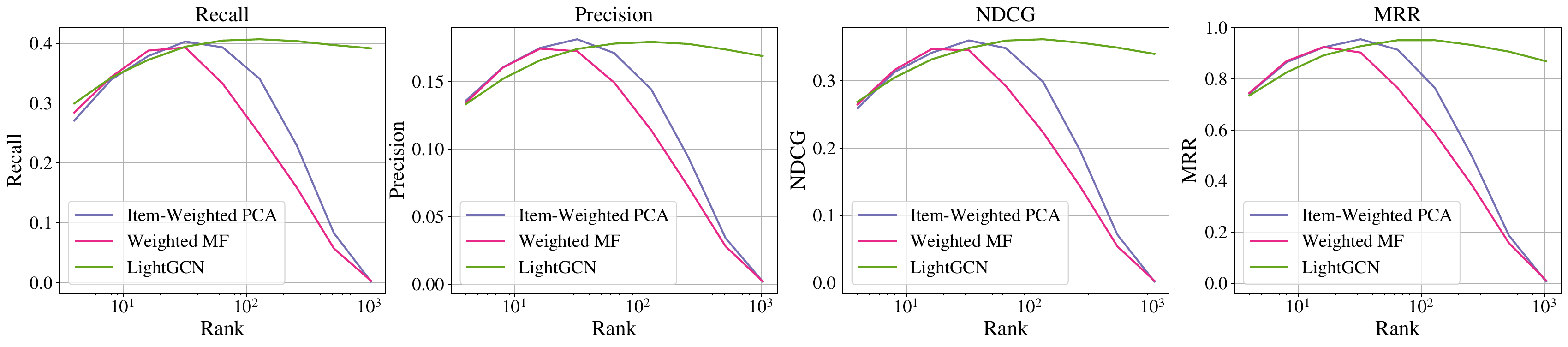}
         \caption{MovieLens}
         \label{fig:movielens_all_baselines}
     \end{subfigure}
     \Description{A figure comparing Item-Weighted PCA with more advanced baselines. For each dataset, there is a line plot for each downstream evaluation metric. For LastFM, Item-Weighted PCA is compared against Weighted Matrix Factorization, and for MovieLens, our algorithm is compared against Weighted Matrix Factorization and LightGCN.}
        \caption{Compared to more advanced baselines such as non-convex matrix factorization and LightGCN, \algname{} yields comparable performance. As LightGCN is applicable only to binary interaction data, we include results only for MovieLens. At higher rank values, performance for \algname{} and Weighted Matrix Factorization (MF) decreases likely due to overfitting, which LightGCN avoids.}
        \label{fig:all_baselines}
\end{figure*}
We also compare the downstream user performance of \algname{} against that of more advanced baselines. Specifically, we compare against the implicit matrix factorization approach of \citet{hu2008collaborative} as well as a deep CF model in LightGCN \cite{he2020lightgcn}. The objective function in \citet{hu2008collaborative} differs from PCA in that the objective acknowledges many values in $X$ are missing, and assigns higher ``confidence'' to observed values. We utilize the implementation of \citet{hu2008collaborative} available in the \texttt{implicit} package with $\alpha=0$. On the other hand, LightGCN is an example of a more recent message-passing-based deep CF model. 

Fig.~\ref{fig:all_baselines} includes the results for the supplemental baselines, where \citet{hu2008collaborative} is represented by ``Weighted MF''. We only include LightGCN results for MovieLens because the model is designed for binary interaction data. Across all four metrics and both graphs, \algname{} maintains performance on par with the baselines. The main difference is that in the high-rank regime, LightGCN does not overfit to the training data while the shallow methods do.

\subsection{Robustness analysis} \label{sec:robustness}
We evaluate the performance of \algname{} for various values of $\gamma$. Fig. \ref{fig:gamma_sensitivity} shows the Recall@$20$ of \algname{} on the MovieLens dataset for fixed $r=32$ and $\gamma\in[-2, 0]$ compared to the PCA baselines. Decreasing $\gamma$ from $0$ (Vanilla PCA) initially improves performance (perhaps due to better-inferred item similarities following re-weighting) before harming performance (perhaps over-compensating and over-up-weighting less popular items). The results for $\gamma=-1$ correspond to the results reported in \cref{fig:movielens_missing_data}. As expected from Theorem \ref{thm:instantiations}, the performance of Item-Weighted PCA at $\gamma=0$ and PCA are equal; the performance of \algname{} at $\gamma=-2$ and Normalized PCA are not equal because the data matrix is not an item-regular block-diagonal matrix.
\begin{figure}[ht]
    \centering
    \includegraphics[width=0.8\linewidth]{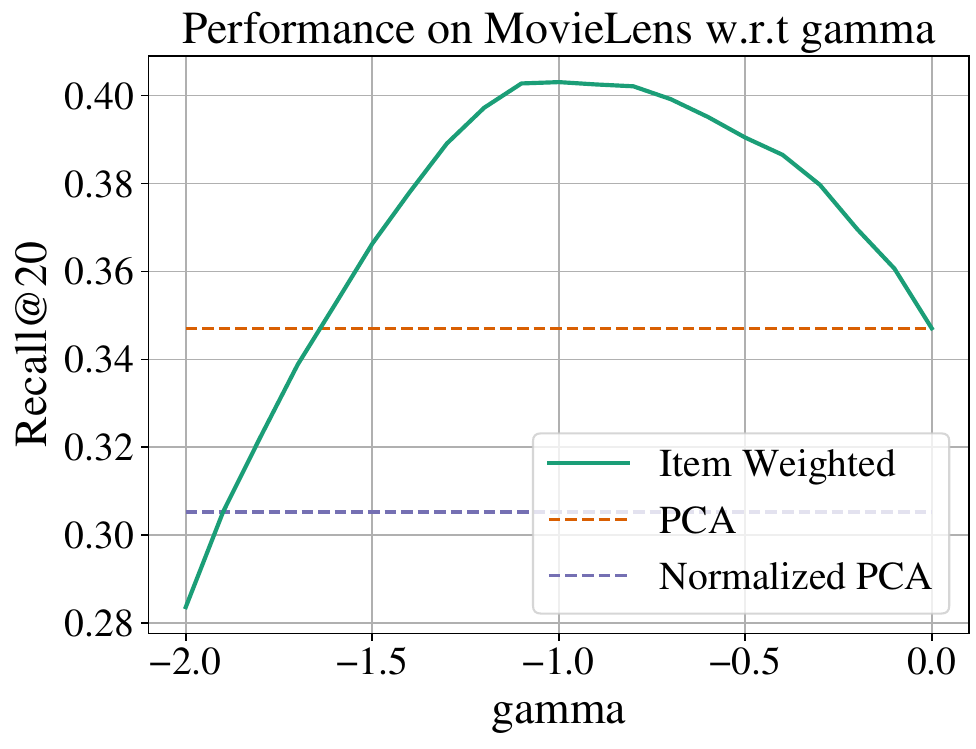}
    \Description {Robustness of Item-Weighted PCA to changes in gamma on the MovieLens dataset. The x-axis is gamma, ranging from -2 to 0, and the y-axis is Recall@20. The performance curve for Item-Weighted PCA peaks at gamma = -1. Horizontal dashed lines indicate the recall values for PCA and Normalized PCA, which do not depend on gamma.}
    \caption{We sweep possible values of $\gamma$ (x-axis) for \algname{} while fixing $r=32$  and report the user-level recall on the hold-out set (y-axis) for MovieLens. The dashed lines correspond to the PCA baselines. On the MovieLens dataset, \algname{} performs best at $\gamma=-1$, which corresponds to the results reported in the main body.}
    \label{fig:gamma_sensitivity}
\end{figure}

\end{document}